\documentclass[10pt]{article}

\PassOptionsToPackage{dvipsnames}{xcolor}
\usepackage[accepted]{rlj}
\usepackage{amssymb,mathtools,amsthm,bm}
\usepackage{xspace}
\usepackage{tikz}
\usetikzlibrary{arrows.meta,positioning,calc}
\usepackage{thm-restate}

\definecolor{condD}{RGB}{44,123,182}
\definecolor{condR}{RGB}{215,48,39}
\definecolor{condF}{RGB}{45,122,30}
\definecolor{condM}{RGB}{117,82,160}
\definecolor{bellmanfill}{RGB}{255,247,230}
\colorlet{mixOS}{condD!50!condR}
\colorlet{mixAS}{condD!50!condF}
\colorlet{mixCS}{condR!50!condF}

\DeclareMathOperator*{\ext}{ext}

\def\equationautorefname#1#2\null{Eq.#1(#2\null)}

\newtheorem{definition}{Definition}

\newtheorem{lemma}{Lemma}
\newtheorem{proposition}{Proposition}
\newtheorem{corollary}{Corollary}

\newcommand{\condD}{\hyperref[cond:D]{Condition~D}\xspace}
\newcommand{\condR}{\hyperref[cond:R]{Condition~R}\xspace}
\newcommand{\condA}{\hyperref[cond:A]{Condition~A}\xspace}

\hypersetup{
  linkcolor  = BrickRed,
  citecolor  = MidnightBlue,
  urlcolor   = MidnightBlue,
  colorlinks = true,
}

\title{Generalised Bellman recurrence and three dualities in sequential decision-making}
\setrunningtitle{Generalised Bellman recurrence and three dualities in sequential decision-making}

\author{Fernando E. Rosas\textsuperscript{1-3}, David Hyland\textsuperscript{4}, and Daniel Polani\textsuperscript{5}
\\}

\emails{\small{f.rosas@sussex.ac.uk, david.hyland@cs.ox.ac.uk, d.polani@herts.ac.uk}
}

\affiliations{
$^{1}$\textbf{Department of Informatics, University of Sussex}\\
$^{2}$\textbf{Department of Brain Sciences and Centre for Complexity Science, Imperial College London}\\
$^{3}$\textbf{Centre for Eudaimonia and Human Flourishing, University of Oxford}\\
$^{4}$\textbf{Department of Computer Science, University of Oxford}\\
$^{5}$\textbf{Department of Computer Science, University of Hertfordshire}
}

\begin{document}
\maketitle

\begin{abstract}

What gives the Bellman equation its form? We show that the recursive properties of optimal value functions follow from three conditions: that the dynamics decomposes through sufficient statistics, that the return decomposes recursively, and that the aggregation of uncertainty is compatible with both. When all three conditions hold on a common state, the Bellman equation arises from their mutual consistency; when one fails, tractability can often be recovered by augmenting the state or by deforming return or dynamics. The same conditions are shown to give rise to three dualities: one between probability and return, one between return and aggregation, and one between aggregation and probability. 
Our framework reveals these dualities as arising from a single construction, unifying methods developed separately across reinforcement learning, control, and decision theory.
\end{abstract}

\section{Introduction}
\label{sec:introduction}

The recurrence relations satisfied by value functions and the fixed-point properties of optimal solutions lie at the heart of reinforcement learning. 
These properties are usually traced to the \emph{principle of optimality}, which states that optimal plans are made of optimal sub-plans \citep{Bellman1957}. 
These ideas are so central, and their applicability so broad, that Bellman recursion is often treated as the canonical expression of optimal decision-making in sequential settings. 

In this paper, we investigate the conditions that make the principle of optimality hold. 
We frame the answer as three requirements: 
that dynamics decompose through sufficient statistics (\condD), 
that returns decompose recursively (\condR), 
and that the aggregation of uncertainty is compatible with both (\condA). 
We then show that the Bellman equations arise when all three conditions are satisfied on a common state, and that many generalisations correspond to relaxing one condition and re-establishing it through state augmentation or by deforming returns or dynamics. 

These three conditions are found to give rise to three dualities, each holding one ingredient fixed while exchanging the roles of the other two. 
Specific instances of these dualities have been investigated by separate literatures --- control-as-inference, robust and regularised control, and risk-sensitive control --- yet the connections between these dualities remain unexplored. 
Our framework shows how these symmetries arise from a single construction, and how Bellman recursion results from their mutual compatibility. 
This leads to a unifying language that organises multiple frameworks across the reinforcement learning, control, and decision theory literatures, revealing connections between seemingly unrelated approaches to sequential decision-making.

\section{A general framework for sequential decision problems}
\label{sec:general-framework}

\begin{quote}
    \textit{An optimal policy has the property that whatever the initial state and initial decision are, the remaining decisions must constitute an optimal policy with regard to the state resulting from the first decision.} \citep{Bellman1957}. 
\end{quote}

Bellman’s principle of optimality is often taken as a starting point for deriving optimality relations. 
Instead, the aim of this paper is to formalise the conditions under which this principle holds and to explore the implications of those conditions.

To do this, let us consider an agent taking actions $a\in\mathcal{A}$ in a scenario with states $s\in\mathcal{S}$.\footnote{Uppercase letters (e.g. $X, Y$) are used to denote random variables and lowercase (e.g. $x, y$) their realisations, calligraphic letters (e.g. $\mathcal X ,\mathcal Y$) denote the sets over which they take values, the symbol $\Delta$ (as in $\Delta(\mathcal X)$) denotes the collection of all distributions over those sets, and boldface letters (e.g. $\mathbf{F}$) denote operators.}  
We denote decision histories as $h_{t}=(s_0,a_0,\dots,a_{t-1},s_t)$ in $\mathcal{H} := \mathcal S\times(\mathcal{A} \times \mathcal{S})^*$, and infinite trajectories as $h = (s_0, a_0, s_1, a_1, \ldots)$ in $\mathcal{T} := (\mathcal{S} \times \mathcal{A})^{\mathbb{N}}$. 
Let $\Pi$ be the class of all policies $\pi$ of the form $\pi:\mathcal{H}\to\Delta(\mathcal A)$, so that $A_t\sim\pi(h_{t})$ maps histories into distributions over actions.\footnote{By doing this, we follow a time-indexing convention such that $s_t$ happens before $a_t$.}   
Each policy $\pi$, together with the environment's dynamics and initial distribution over $S_0$, induces uncertainty over trajectories described by an \emph{\textbf{uncertainty functional}} $\mathbf{P}_\pi:\mathcal T\to\mathbb R$.

In this setting, we consider the following decision problem:
\begin{equation}
\label{eq:general-target}
\pi^{*} \;\in\; 
\arg\max_{\pi \in \Pi}
\; \mathbf{F}\Big[\, \mathbf{J}(H)
\;;\;
H\sim \mathbf{P}_\pi
\Big].
\end{equation}
Here, $\mathbf{J} : \mathcal{T} \to \mathcal{U}$ is a \textit{\textbf{return functional}} encoding preferences by assigning returns $u\in\mathcal{U}$ to each trajectory, where $\mathcal{U}$ is a totally or partially ordered set. 
Additionally, $\mathbf{F}:\Delta(\mathcal{U}) \to \mathbb{R}$ is an \textit{\textbf{aggregation functional}} that specifies how uncertainty is treated, mapping the distribution $\mathbf{J}(H)$ under $\mathbf{P}_\pi$ to a real number. 
Thus, the triple $\langle \mathbf{P}_\pi, \mathbf{J}, \mathbf{F}\rangle$ encodes 
\emph{how trajectories are generated}, \emph{how each trajectory is evaluated},
and \emph{how those evaluations are aggregated}, respectively.

On its own, \autoref{eq:general-target} imposes no
conditions on $\mathbf{P}_\pi$, 
$\mathbf{J}$, or $\mathbf{F}$. It
subsumes finite- and infinite-horizon problems, deterministic and
stochastic dynamics, additive and non-additive utilities, and risk-neutral,
risk-sensitive, and robust aggregations. 
The price of this generality is computational: to evaluate the objective in \autoref{eq:general-target} requires aggregating $\mathbf{J}$ over the entire trajectory distribution, and optimising it requires doing so for every policy $\pi\in\Pi$. 
The rest of this section identifies structural assumptions under which this problem reduces to a local recursion, leading to Bellman's principle.

\subsection{Three conditions for Bellman recursion}
\label{subsec:three-conditions}

To make \autoref{eq:general-target} tractable, 
we now formulate conditions on the statistics $\mathbf{P}_\pi$, returns $\mathbf{J}$, and aggregation $\mathbf{F}$ under which the
optimisation admits a Bellman recursion. 
We state each in turn before combining them in \autoref{subsec:bellman-recovery}. 
Below, we consider functions over histories $\phi(h_{t})$ that are \emph{unifilar} (i.e., recursively updatable), so that $\phi(h_{t+1})$ is determined from $\phi(h_{t})$, $a_{t}$, and $s_{t+1}$.

\paragraph{Condition D (history sufficiency).}\label{cond:D}
There exists a unifilar map $\phi_D : \mathcal{H} \to \mathcal{Z}_D$ into
a dynamics latent space $\mathcal{Z}_D$ and a suitable kernel $K:\mathcal{Z}_D\times\mathcal A\to \Delta(\mathcal{S})$  such that the distribution of trajectories under a policy $\pi$ 
factorises as 
\begin{align}
\label{eq:cond-D1}
\mathbf{P}_\pi(h_t) 
&= 
\rho(s_0)
\prod_{\tau= 0}^{t-1}
\pi(a_\tau \,\big|\, h_{\tau})\;
K\bigl(s_{\tau+1} \,\big|\, \phi_D(h_{\tau}),a_\tau\bigr),
\end{align}
with $\rho$ a distribution over $S_0$. 
This implies that the one-step transition depends on the history only through
the sufficient statistic $\phi_D(h_{t})=z_t$. 
The Markov property is the special case where $z_t = s_{t}$. In partially observed settings, this condition still holds by setting $z_t$ to be a belief state~\citep{KaelblingLittmanCassandra1998}.  
Thus, \condD is always met (even in non-stationary settings), but it becomes non-trivial when considering non-probabilistic notions of uncertainty (see \autoref{app:condition-d}).

\paragraph{Condition R (return compositionality).}\label{cond:R}
There exist a unifilar map $\phi_R : \mathcal{H} \to \mathcal{Z}_R$ into
a return latent space $\mathcal{Z}_R$, a local return function
$\ell : \mathcal{Z}_R \times \mathcal{A} \times \mathcal{S} \to \mathcal{U}$,
and an associative binary operator $\oplus : \mathcal{U} \times
\mathcal{U} \to \mathcal{U}$ such that the return functional admits the decomposition
\begin{equation}
\label{eq:cond-R}
\mathbf{J}(h) \;=\; \bigoplus_{\tau \ge 0}
\ell\bigl(\phi_R(h_{\tau}),\, a_\tau,\, s_{\tau+1}\bigr),
\end{equation}
where $\bigoplus_{\tau \ge 0} A_\tau := A_0 \oplus A_1 \oplus \dots$. 
This condition implies that the one-step return depends on the history only through the sufficient statistic $\phi_R(h_{t})$. 
For example, if $\ell$ depends only on $(s_t, a_t)$
and there is no time dependence, $\phi_R$ can be taken to be a constant and \autoref{eq:cond-R} leads to memoryless rewards. 
In contrast, $\ell_t = \gamma^t r(s_t, a_t)$ requires $\phi_R$ to track the running
discount $\gamma^t$.

\paragraph{Condition A (aggregation compatibility).}\label{cond:A}
The aggregation functional $\mathbf{F}$ must be compatible with the 
compositions induced by \condD 
and \condR:

\begin{itemize}
\item \emph{Kernel-composition compatibility (A1)}: For any kernels $K_1 : \mathcal{X} \to \Delta(\mathcal{Y})$ and
$K_2 : \mathcal{Y} \to \Delta(\mathcal{Z})$,
\begin{equation}
\label{eq:cond-A1}
\mathbf{F}_{Z\sim K_2\circ K_1(\cdot|x)}\big[Z\big] 
=
\mathbf{F}_{Y\sim K_1(\cdot|x)}\Big[
\mathbf{F}_{Z\sim K_2(\cdot|Y)}\big[
Z
\big]
\Big].
\end{equation}
Thus, the map $K \mapsto \mathbf{F}_K$ satisfies 
$\mathbf{F}_{K_2 \circ K_1} \;=\; \mathbf{F}_{K_1} \circ \mathbf{F}_{K_2}$.\footnote{For convenience, we sometimes use the shorthand notation $\mathbf{F}_{X\sim K}[X]:=\mathbf{F}[X\,;\,X\sim K]$, where $K\in\Delta(\mathcal X)$.}

\item \emph{Return compatibility (A2)}: 
There exists an binary operator $\odot : \mathcal{U} \times
\mathbb{R} \to \mathbb{R}$ such that\begin{equation}
\mathbf F_{Y\sim K}
\bigl[
u\oplus Y
\bigr]
=
u\odot
\mathbf F_{Y\sim K}[Y],
\qquad
\forall u\in\mathcal{U}.
\end{equation}
Thus, the aggregation functional respects return composition for non-random terms. 
\end{itemize}

If $\mathbf{F} = \mathbb{E}$ is the expectation and $\oplus = \odot = +$, then 
(A1) is satisfied by the tower property of expectation and (A2) is satisfied by linearity. 
More generally, \condA is also satisfied by the log-sum-exp aggregation $\hat{\mathbf{F}}^\beta[Z] = \beta^{-1}\log\mathbb{E}[\exp(\beta Z)]$~\citep{morales2023thermodynamics}. 

It is useful to note that \condD and \condR can be expressed recursively. \autoref{eq:cond-D1} can be expressed as the following two-stage factorisation
\begin{equation}
\label{eq:cond_dist}
\mathbf{P}_\pi(h \mid h_{t}) 
\;=\;
\Big[\pi(a_t\mid h_{t}) K(s_{t+1} \mid \phi_D(h_{t}),a_t) 
\Big]
\times \mathbf{P}_\pi(h \mid h_{t+1}).
\end{equation}
Similarly, \autoref{eq:cond-R} can be expressed as
\begin{equation}
\label{eq:cond-R-recursion}
\mathbf{J}^{(t)}(h) \;=\;
\ell\bigl(\phi_R(h_{t}),\, a_t,\, s_{t+1}\bigr) \;\oplus\;
\mathbf{J}^{(t+1)}(h),
\end{equation}
where $\mathbf{J}^{(t)}(h) := \bigoplus_{\tau \ge t} \ell(\phi_R(h_{\tau}), a_{\tau}, s_{\tau+1})$
is the tail return from step $t$, with $\mathbf{J}^{(0)} = \mathbf{J}$.

\subsection{Value function and Bellman recursion}
\label{subsec:bellman-recovery}

Let us now introduce value functions for this general setting. 

\begin{definition}
\label{def:history-value}
The value of history $h_{t} \in \mathcal{H}$ under policy $\pi \in \Pi$ is
\begin{equation}
\label{eq:history-value}
V_\pi\big(h_{t}\big) 
\;:=\; 
\mathbf{F}
\Big[\,
\mathbf{J}^{(t)}\big(H\big) 
\,;\, 
H \sim \mathbf{P}_\pi(\,\cdot \mid h_{t})
\,
\Big],
\end{equation}
where 
$\mathbf{P}_\pi(h| h_{t})$ is the distribution defined in \autoref{eq:cond_dist} and
$\mathbf{J}^{(t)}$ is the
tail return as defined in \autoref{eq:cond-R-recursion}. 
\end{definition}

We now show how the above conditions jointly give rise to Bellman recursion. 
For this, let us define the \emph{Bellman state} $z_t\in \mathcal{Z}:=\mathcal{Z}_D\times\mathcal{Z}_R$ as 
\begin{equation}
\label{eq:bellman-state}
z_t^D := \phi_D(h_{t}),
\qquad
z_t^R := \phi_R(h_{t}),
\qquad
z_t := (z_t^D,z_t^R).
\end{equation}
Let us also introduce $\phi:=(\phi_D,\phi_R)$ with $\phi:\mathcal{H}\to\mathcal{Z}$,
so that $z_t = \phi(h_{t})$. 
The unifilarity of $\phi_D$ and $\phi_R$ (assumed in \condD and \condR) imply that $\phi$ itself is unifilar, so that there exists a map
$\nu : \mathcal{Z} \times \mathcal{A} \times \mathcal{S} \to \mathcal{Z}$
such that
\begin{equation}
\label{eq:phi-forward-closure}
\phi(h_{t+1}) \;=\; 
\nu\bigl(\phi(h_{t}),\, a_{t},\, s_{t+1}\bigr),
\qquad \forall h_{t} \in \mathcal{H},\, a_t \in \mathcal{A},\, s_{t+1} \in \mathcal{S}.
\end{equation}
We consider the subclass of policies of the form
\begin{equation}
\Pi_\phi \;:=\; \Bigl\{\pi \in \Pi \,:\,
\phi(h_{t}) = \phi(h'_{t}) \,\Longrightarrow\,
\pi(\cdot \mid h_{t}) = \pi(\cdot \mid h'_{t}), \forall h_{t},h'_{t},t\geq 0
\Bigr\}.
\end{equation}
Each $\pi\in\Pi_\phi$ factors through $\phi$ as $\pi(\cdot|h_{t}) = \bar{\pi}(\cdot|z_t)$ for a unique $\bar{\pi}:\mathcal{Z}\to\Delta(\mathcal{A})$. 
If \condD and \condR hold, then $\phi(h_{t})=\phi(h'_{t})$ implies
$V_\pi(h_{t})=V_\pi(h'_{t})$ for all $\pi\in\Pi_\phi$ (see \autoref{lem:value-factorisation}, proved in App.~\ref{sec:appendix_proof_lemma}). This implies that $V_\pi$ also factors through $\phi$ as $V_\pi(h_{t})=\bar V_\pi(\phi(h_{t}))$ for a unique $\bar V_\pi:\mathcal{Z}\to\mathbb{R}$. 
This leads to a corresponding Q-function
\begin{equation}
\label{eq:q-function}
\big(\mathbf{B} \bar V\big)(z,a)
\;:=\;
\mathbf F_{S\sim K(\cdot\mid z^D,a)}
\Big[
\ell(z^R,a,S)\odot \bar{V} \big(\nu(z,a,S)\big)
\Big],
\end{equation}
with $\nu$ as defined in \autoref{eq:phi-forward-closure}. 
One can use this to define $(\mathbf{T}_\pi\bar V)(z): = \mathbf{F}_{
A\sim\bar{\pi}(\cdot \mid z)
}
\!\bigl[\, 
\big(\mathbf{B} \bar V\big)(z,A)
\,\bigr]$ as a generalised policy evaluation operator for $\pi\in\Pi_\phi$.

\begin{restatable}[Generalised Bellman recursion]{theorem}{generalisedbellman}
\label{thm:generalised-bellman}

Assume \condD, \condR, and \condA hold. 
Then, $\bar V_\pi$ for $\pi \in \Pi_\phi$
satisfies the following fixed-point condition:
\begin{equation}
\label{eq:generalised-bellman}
\bar V_\pi(z) 
\;=\;
(\mathbf{T}_\pi\bar V_\pi)(z),
\quad\forall z\in\mathcal{Z}.
\end{equation}
\end{restatable}

\begin{corollary}[Bellman expectation equation]
\label{cor:classical-bellman-expectation}
Applying \autoref{thm:generalised-bellman} to
$\mathbf{F} = \mathbb{E}$, $\oplus = \odot = +$,
$\phi_D(h_{t}) = s_{t}$,
$\phi_R(h_{t}) = (s_{t}, \gamma^t)$,
and $\ell\bigl((s, \gamma^t), a, s'\bigr) = \gamma^t\, r(s, a)$ reduces  \autoref{eq:generalised-bellman} to
\begin{equation}
\label{eq:classical-bellman-expectation}
\bar V_\pi(s) \;=\; \mathbb{E}_{A \sim \pi(\cdot \mid s),\; S' \sim K(\cdot \mid s, A)}\!\Bigl[\,r(s, A) \;+\; \gamma\, \bar V_\pi(S')\,\Bigr],
\quad
\forall s\in\mathcal{S}.
\end{equation}
\end{corollary}

\subsection{Generalised Bellman optimality}

To go from recursion to optimality, let us define the \emph{generalised Bellman optimality} operator
\begin{equation}
\label{eq:bellman-operator}
    (\mathbf{T}_*\bar V)(z): =     
    \sup_{\mu\in\Delta(\mathcal A)}
    \mathbf{F}_{
    A\sim\mu
    }
    \Bigl[\, 
    \big(\mathbf{B} \bar V\big)(z,A)
    \,\Bigr].
\end{equation}
We now focus on aggregator operators that are monotonic, which satisfy
\begin{equation}
\label{eq:cond-M}
f(u) \le g(u), \;\forall u\in\mathcal U
\;\Longrightarrow\;
\mathbf{F}_{U \sim K}\big[\,f(U)\,\big]
\;\le\;
\mathbf{F}_{U \sim K}\big[\,g(U)\,\big]
\end{equation}
for  
functions $f, g : \mathcal{U} \to \mathcal{U}$. 
This condition is satisfied by many aggregators of interest 
(expectation, log-sum-exp, sup, inf, robust min-max).
We also focus on monotonic operations $\odot$ satisfying
\begin{equation}
    x\leq y 
    \;\Longrightarrow\;
u\odot x\leq
u\odot y,
    \quad
    \forall u\in\mathcal{U}.
\end{equation}
Define the optimal value function $V^*(h_{t}):=\sup_{\pi\in\Pi}V_\pi(h_{t})$. By
\autoref{lem:opt-factorisation}, 
there exists a function $\bar V^*:\mathcal Z\to\mathbb R$ such that 
$V^*(h_{t}) = \bar V^*\big(\phi(h_{t})\big)$. 

\begin{restatable}[Generalised Bellman optimality]{theorem}{bellmanoptimality}
\label{thm:bellman-optimality}
Assume \condD, \condR, and \condA hold, and $F$ and $\odot$ are
monotonic. If $\bar V^*$ is attained, then
\begin{equation}
\label{eq:bellman_optimality}
  \bar V^*(z) 
  = 
  (\mathbf{T}_*\bar V^*)(z),
  \qquad \forall z\in Z .
\end{equation}
\end{restatable}

\begin{corollary}[Bellman optimality equation]
\label{cor:classical-bellman}
If $\mathbf{F} = \mathbb{E}$, $\oplus = +$, 
$\phi_D(h_{t}) = s_t$, $\phi_R(h_{t}) = (s_t,\gamma^t)$, and
$\ell\big((s,\gamma^t), a, s'\big) = \gamma^t r(s, a)$, then \autoref{eq:bellman_optimality} reduces to
\begin{equation}
\label{eq:classical-bellman}
\bar{V}^{*}(s) \;=\; \sup_{a \in \mathcal{A}}\, \Bigl\{\,
r(s, a) \,+\, \gamma\, \mathbb{E}_{S' \sim K(\cdot \mid s, a)}\big[\, \bar{V}^{*}(S')\,\big]
\,\Bigr\}.
\end{equation}
\end{corollary}

\begin{figure}
    \centering
    \begin{tikzpicture}[
  >=Stealth, font=\small,
  vtx/.style={circle, draw=#1, line width=1.4pt, fill=#1!17, align=center,
              minimum size=1.2cm, inner sep=1pt, font=\large},
  pill/.style={rounded corners=4pt, fill=#1!2, draw=#1, line width=1pt,
               text=#1, font=\footnotesize, align=center, inner sep=3pt},
  clab/.style={font=\scriptsize, align=center},
  dl/.style={line width=2pt, #1, shorten >=2pt, shorten <=2pt},
]
\def\R{3.15}
\node[vtx=condD] (P) at (90:\R/2)   {$\mathbf{P}_\pi$};
\node[vtx=condR] (J) at (210:\R)  {$\mathbf{J}$};
\node[vtx=condF] (F) at (-30:\R)  {$\mathbf{F}$};
\node[clab, text=condD, above=1pt of P] {\textbf{Condition D}\\ history sufficiency};
\node[clab, text=condR, below=1pt of J] {\textbf{Condition R}\\ return compositionality};
\node[clab, text=condF, below=1pt of F] {\textbf{Condition A}\\ aggregation compatibility};

\draw[dl=mixOS,<->] (P) -- (J);   \draw[dl=mixAS,<->] (P) -- (F);   \draw[dl=mixCS,<->] (J) -- (F);   

\node[pill=mixOS, anchor=east] at ($(P)!0.5!(J)+(-0.55,0.30)$)
      {\textbf{Object duality}\\[1pt] $\mathbf{P}_\pi \!\leftrightarrow\! \mathbf{J}$\\[1pt]
       {\scriptsize($\mathbf{F}=\mathbb{E}$ fixed)}};
\node[pill=mixAS, anchor=west] at ($(P)!0.5!(F)+(0.55,0.30)$)
      {\textbf{Representation duality}\\[1pt] $\mathbf{P}'_\pi \!\leftrightarrow\! \tilde{\mathbf{F}}^g$\\[1pt]
       {\scriptsize($\mathbf{J}$ fixed)}};
\node[pill=mixCS] at ($(J)!0.5!(F)+(0,-1.0)$)
      {\textbf{Assessment duality}\\[1pt] 
      $u \circ \mathbf{J} \!\leftrightarrow\! \hat{\mathbf{F}}^u$\\[1pt]
      $\mathbf{J}+\text{Reg} \!\leftrightarrow\! \hat{\mathbf{F}}^\beta$\\[1pt]
       {\scriptsize($\mathbf{P}_\pi$ fixed)}};
\end{tikzpicture}
    \caption{\small{\textbf{Three dualities in reinforcement learning. } These can be arranged as a triangle, where each edge is a duality holding the opposite vertex fixed and exchanging the other two. \emph{Object duality} relates the dynamics and return using the bilinearity of the expectation; 
    \emph{assessment duality} exchanges non-linearity of aggregation and return curvature or regularisation; and \emph{representation duality} exchanges deformations of the probabilities.}}\label{fig:duality}
\end{figure}
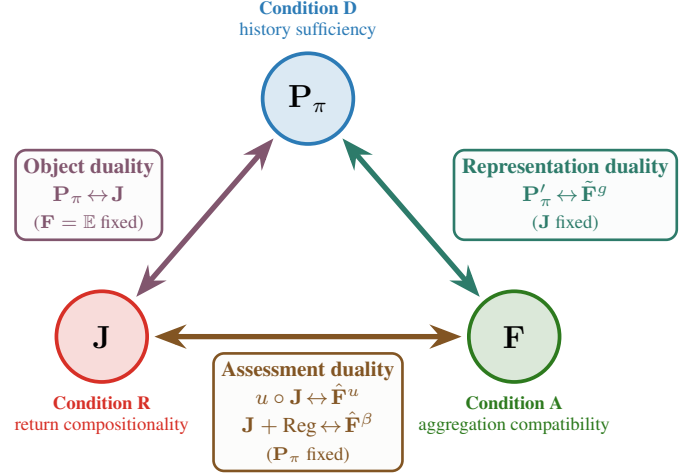

\section{Three dualities in the Bellman structure}
\label{sec:three-fold}

After clarifying how the principle of optimality arises from conditions over dynamics, returns, and aggregation, a natural next question is whether these conditions play comparable roles in the resulting recursion. 
Here we investigate three dualities between these elements, which are illustrated in \autoref{fig:duality}. 
For simplicity, we focus on the case of $\mathcal{U}=\mathbb{R}$ and $\oplus=\odot=+$.

\subsection{Object duality: statistics--vs--returns}
\label{subsec:log-transform}

Classical decision theory describes preferences as arising from two distinct entities: probabilities representing \emph{beliefs} and returns representing \emph{desires} \citep{von1944theory,savage1954foundations}. 
Despite these objects being of different kinds, they are both treated linearly by the expectation operator. 
This symmetry allows us to reparameterise the roles of probabilities and returns in \autoref{eq:general-target} while keeping the preferences unchanged --- a degree of freedom we call \emph{object duality}. 

To study this, let us first note that using $\mathbf{F}=\mathbb{E}$ as aggregator results in a bilinear pairing of the form
\begin{equation}\label{eq:opt_linear}
    \mathbb{E}_{H\sim\mathbf{P}_\pi}
    \Big[\mathbf{J}(H)\Big] = \sum_{h\in\mathcal{T}} 
    \mathbf{J}(h) \mathbf{P}_\pi(h)
    = 
    \sum_{i\in\mathcal I} u_i p_i
    = 
    \big\langle \bm{u}, \bm{p} \big\rangle,
\end{equation}
where $\mathcal{I}$ indexes the possible trajectories. 
Moreover, $\big\langle \bm{u}, \bm{p} \big\rangle = \big\langle \bm{p}, \bm{u} \big\rangle$, which leverages the natural duality between $\Delta(\mathcal T)$ and $\mathbb R^{\mathcal T}$. In particular, if $u_i\geq 0$ and $\kappa = \sum_i u_i \in (0,\infty)$, then maximising the expected return $\mathbf{J}$ under $\mathbf{P}_\pi$ is equivalent to maximising the return $\kappa\cdot\mathbf{P}_\pi$ under $\mathbf{J}/\kappa$.

A richer way to intermix probabilities and returns is by applying the following rotation:
\begin{equation}
    p_i\to p_i'=\frac{g(u_i) p_i}{\sum_i g(u_i) p_i},
    \qquad
    u_i\to u_i'= \frac{u_i}{g(u_i)},
\end{equation}
where it is assumed that $g(u_i)>0$ for all $u_i$.  
For an affine map $g(u)=\alpha + \beta u$, one finds that
\begin{equation}
    \big\langle \bm{u}', \bm{p}' \big\rangle
    =
    \frac{\big\langle \bm{u}, \bm{p} \big\rangle}
    {\alpha + \beta\big\langle \bm{u}, \bm{p} \big\rangle}
    . 
\end{equation}
As the mapping $x\to x/(\alpha+\beta x)$ is increasing for $\alpha>0$, this implies that an agent maximising $u_i$ under $p_i$ is indistinguishable from one maximising $u'_i$ under $p'_i$. 
For $\alpha=1$ this corresponds to the so-called Jeffrey--Bolker rotations \citep{Jeffrey1965,Bolker1966,complex2018}.

The symmetry between returns and probabilities can also be exploited by interpreting the return as arising from a distribution. This can be done by introducing the distribution $q_i := e^{u_i}/Z$ with $Z=\sum_j e^{u_j}$, which allows us to rewrite \autoref{eq:opt_linear} as \citep{ortega2016memory,Wentworth2021UtilityCompression}
\begin{equation}
\label{eq:MDL}
    \mathbb{E}_{H\sim\mathbf{P}_\pi}
    \Big[\mathbf{J}(H)\Big] 
    = 
    \sum_i p_i \log q_i + \log Z
    = 
    - \underbrace{ H(p;q)}_{\text{cross-entropy}}
    + 
    \underbrace{\log Z}_{\text{free energy}}.
\end{equation}
Thus, re-expressing the return as a log-probability turns a problem of expected return maximisation into a minimisation of description length \citep{grunwald2007minimum}. 
This serves as the foundation for active inference~\citep{Friston2015,DaCosta2020}, which replaces the Helmholtz free energy in \autoref{eq:MDL} with variational free energy to pursue approximate Bayesian inference (see \autoref{app:active_inference}).

\subsection{Assessment duality: return--vs--aggregation}
\label{subsec:return-aggregation-duality}

An agent's attitude towards uncertainty is reflected by the aggregation functional $\mathbf{F}$, which may operate non-linearly on the return. Here we study how such non-linear effects can be relocated from the aggregation into the returns --- a degree of freedom that we call \emph{assessment duality}. 

To investigate this, consider an increasing 
utility transformation \(u:\mathbb R\to\mathbb R\). Given a return functional \(\mathbf J(h)\),
one may either transform the return itself and aggregate linearly, 
$\mathbb E_{H\sim \mathbf P_\pi}\big[u\big(\mathbf J(H)\big)\big]$, 
or leave the return unchanged and replace expectation by the
certainty-equivalent aggregator
\begin{equation}
\label{eq:certainty-equivalent-aggregator}
\hat{\mathbf F}^u[Z]
:=
u^{-1}\Big(
\mathbb E\big[u(Z)\big]
\Big).
\end{equation}
For example, if $u$ is concave then Jensen's inequality gives $\hat{\mathbf{F}}^u[Z] \leq \mathbb{E}[Z]$, so $\hat{\mathbf{F}}^u$ penalises uncertainty. 
Since \(u^{-1}\) is also strictly increasing, the two objectives induce the same
policy ordering:
\begin{equation}
\label{eq:return-aggregation-equivalence}
\arg\max_{\pi}
\mathbb E_{h\sim \mathbf{P}_\pi}
\Big[
u\big(\mathbf J(h)\big)
\Big]
=
\arg\max_{\pi}
\hat{\mathbf F}^u_{h\sim \mathbf{P}_\pi}
\Big[
\mathbf J(h)
\Big].
\end{equation}
Regarding \condA, $\hat{\mathbf{F}}^u$ satisfies A1 but only satisfies A2 for the log-sum-exp family given by $u(x)=e^{\beta x}$ \citep{morales2023thermodynamics}. If $u(x)$ is such that A2 does not hold (e.g. for the geometric mean), assessment duality can be used to recover Bellman recursion~\citep{maclean2011kelly}.

Thus, the same risk attitude can be represented either by risk-sensitive aggregation functional \(\hat{\mathbf F}^u\), or by modifying the curvature of returns via functional composition \(u\circ \mathbf J\) and aggregating them in a risk-neutral manner. This idea is exploited by risk-sensitive control \citep{HowardMatheson1972,Whittle1981} and multiplicative dynamic programming \citep{Bellman1957,Puterman1994MDPs}. 

A second facet of the assessment duality turns the effects of non-linear aggregation into reward regularisation. 
To study this, note first that the Bellman optimality operator features two distinct aggregations: an expectation over next-state transitions (\autoref{eq:q-function}) and a maximisation over actions (\autoref{eq:bellman-operator}). 
These two operations can be unified by seeing them as particular cases of the log-sum-exp family given by $\hat{\mathbf{F}}^\beta:=\hat{\mathbf{F}}^u$ for $u(x)=e^{\beta x}$, which interpolates between expectation ($\beta \to 0^{+}$) and supremum ($\beta \to \infty$) \citep{ortega2015informationtheoreticboundedrationality}. 
Then, one can generalise \autoref{eq:classical-bellman} as 
\begin{equation}
\label{eq:two-param-bellman}
V^{*}_{\beta_a, \beta_s}(s) \;=\; \hat{\mathbf{F}}^{\beta_a}_{A\sim \pi_0(\cdot|s)}\Bigl[\,
r(s, A) + \gamma\, \hat{\mathbf{F}}^{\beta_s}_{S'\sim K(\cdot|s,A)}\bigl[\, V^{*}_{\beta_a, \beta_s}(S') \,\bigr]
\,\Bigr],
\qquad
\beta_a, \beta_s \in \mathbb{R} \cup \{\pm\infty\},
\end{equation}  
where $\pi_0$ is a reference policy. 
For different values of $(\beta_a,\beta_s)$ this formulation recovers several well-known formulations:
\begin{itemize}
    \item For $(+\infty,-\infty)$, it becomes the \textit{robust Bellman equation} 
\begin{equation}
\label{eq:robust_bellman_equation}
    V^{*}(s) = \max_{a\in\mathcal A} \Big\{r(s,a) + \gamma\, \min_{q \in \mathcal Q(s,a)} \mathbb{E}_{S'\sim q}[V^{*}(S')]\Big\},
    \qquad
    \mathcal Q(s,a)\subset \Delta(S),
\end{equation}
whose fixed-point corresponds to the best policy against an adversarial environment with value $V^*_\text{rob}(s) = \sup_\pi \inf_{\mathbf{Q}_\pi\in \mathcal Q} V^\pi_{\mathbf{Q}_\pi}(s)$ \citep{Iyengar2005,NilimElGhaoui2005}.\footnote{Formulating robust RL within our framework requires a more general notion of uncertainty (see App.~\ref{sec:beyond_probability}).} 
\item For $(+\infty, \beta_s)$, it becomes the \textit{risk-sensitive Bellman equation}
\begin{equation}
\label{eq:risk-sensitive}
V^{*}(s) \;=\; \max_{a\in\mathcal A} \, \Bigl\{\, r(s,a) + \tfrac{1}{\beta_s} \log \mathbb{E}_{S'\sim K(\cdot|s,a)}\big[ \exp\bigl(\beta_s\, V^{*}(S')\bigr) \big] \,\Bigr\},
\end{equation}
whose fixed point is $V^*_\text{rsc}(s) = \sup_\pi \beta_s^{-1}\log\mathbb{E}_\pi[\exp(\beta_s\mathbf{J}^{(t)})|S_t=s]$, being a non-linear average of the return~\citep{HowardMatheson1972,Whittle1981}.\footnote{For small $\beta_s$ the value becomes 
$V^*_\text{rsc}(s) \approx \mathbb{E}_\pi[\mathbf{J}|S_t=s] + \beta_s/2\cdot\text{Var}[\mathbf{J}|S_t=s]$, so that variance is rewarded or penalised depending on the sign of $\beta_s$.} 
\item For $(\beta_a,0)$, it becomes the \textit{soft Bellman equation}
\begin{equation}
\label{eq:soft-bellman-standard}
V^*(s)
=
\frac{1}{\beta_a}
\log
\sum_{a\in\mathcal A}
\pi_0(a\mid s)
\exp\Big(
\beta_a
\left[
r(s,a)
+
\mathbb E_{S'\sim K(\cdot\mid s,a)}
V^*(S')
\right]
\Big),
\end{equation}
whose fixed-point satisfies $V^*_\text{soft}(s) = \sup_\pi \big\{ \mathbb{E}_\pi[\mathbf{J}|S_t=s] - \beta_a^{-1} D_\text{KL}\big(\mathbf P_\pi\|\mathbf P_{\pi_0}\big) \big\}$ \citep{Ziebart2010,Haarnoja2017}. For $\pi_0$ uniform this reduces to the objective in control-as-inference and maximum-entropy RL~\citep{Todorov2007,Kappen2012,Levine2018}, being  closely related to active inference~\citep{Millidge2020AIFvsCAI,dacosta2023reward}.
\end{itemize}

Thanks to assessment duality, these three formulations can also be obtained by regularising the return as
$\mathbf{J}(h)\to\mathbf{J}(h) + \text{(Reg)}$~\citep{petersen2000minimax,osogami2012robustness,husain2021regularized,derman2021twice,brekelmans2022your}. For the robust and risk-sensitive settings one can find that
\begin{equation}
    V^*(s) = \max_{a\in\mathcal{A}} \ext_{q\in\Delta(S)} \mathbb{E}_{S'\sim q}\,\Big[\, \hat{\ell}(s,a,S';q) + \gamma V^*(S')\,\Big]\
\end{equation}
where $\ext=\inf$ for robust and risk-averse cases and $\ext=\sup$ for risk-seeking, with
\begin{subequations}
\begin{align}
    \hat{\ell}_\text{rob}(s,a,s';q) 
   &= 
   r(s,a) + I_{\mathcal Q(s,a)}(q),
   && \text{(robust Bellman)} \label{eq:robust}
   \\ 
   \hat{\ell}_\text{risk}(s,a,s';q) 
   &= 
   r(s,a) + \frac{1}{\beta_s}\log\frac{K(s'|s,a)}{q(s')},
   && \text{(risk-sensitive Bellman)} \label{eq:risk_sensitive}
\end{align}
\end{subequations}
where $I_{\mathcal Q(s,a)}(q)=\infty$ if $q\notin \mathcal Q(s,a)$ and zero otherwise. The soft setting becomes
\begin{align}
    V^*(s) 
    &= \sup_{\mu\in\Delta(\mathcal A)} \mathbb{E}_{A\sim \mu, S'\sim K(\cdot|s,A)}\,\Big[\, \hat{\ell}(s,A;\pi) + \gamma V^*(S')\,\Big] 
    \quad\text{with}\\
   \hat{\ell}_\text{soft}(s,a;\pi) 
   &= 
   r(s,a) - \frac{1}{\beta_a}\log\frac{\mu(a|s)}{\pi_0(a|s)}.
   \label{eq:soft}
\end{align}
Proofs are provided in \autoref{app:first_duality_misc}. 

\subsection{Representation duality: dynamics--vs--aggregation}
\label{sec:pf-duality}

While the previous section considered aggregation functionals that are non-linear over the return, here we consider aggregations that may be non-linear over distributions. In particular, we study how such non-linearities can be relocated into a deformation of the probabilities --- a degree of freedom we call \emph{representation duality}. 

Concretely, we are looking for a deformation of the trajectory law $\mathbf{P}_\pi \mapsto \mathbf{P}_\pi'$ such that
\begin{equation}
\mathbf{F}_{H\sim \mathbf{P}_\pi}\big[ \mathbf{J}(H)\big]
\;=\;
\mathbb{E}_{H\sim \mathbf{P}_\pi'}\big[ \mathbf{J}(H)\big],
\label{eq:deformation-question}
\end{equation}
thus turning a non-linear aggregator into ordinary expectation. 
It can be shown that such a deformation exists for aggregators of the form (see \autoref{app:distortion-deformation})
\begin{equation}
\label{eq:choquet_aggregator}
    \tilde{\mathbf{F}}^g\Big[\,\mathbf{J}(H)\,;\,H\sim\mathbf{P}_\pi\Big]
    :=
    \int_0^\infty g\big(S(x)\big)\text{d}x 
    - 
    \int_{-\infty}^0\Big[1-g\big(S(x)\big)\Big]\text{d}x,
\end{equation}
where $S(x) = \mathbf{P}_\pi\big[\mathbf{J}(H)>x\big]$ and 
$g:[0,1]\to[0,1]$ is an increasing function with $g(0)=0$ and 
$g(1)=1$~\citep{follmer2025stochastic}. 
Representation duality is closely related to three classic decision-theoretic constructions: 
Yaari's dual theory of choice \citep{yaari1987}, Choquet expected utility
\citep{schmeidler1989subjective}, and max-min expected utility
\citep{gilboa1989maxmin}. 

For example, consider the conditional value at risk aggregator 
$\text{CVaR}_\alpha[\mathbf{J}(H)] = \tilde{\mathbf{F}}^g[\mathbf{J}(H);H\sim\mathbf{P}_\pi]$ for 
$g(s) = \alpha^{-1} \max\{ s- (1-\alpha),0\}$ \citep{rockafellar2000optimization}. 
It can be shown that
\begin{equation}
\mathrm{CVaR}_\alpha[\mathbf{J}(H)]
\;=\;
\mathbb{E}_{H\sim \mathbf{P}'_\pi}\big[ \mathbf{J}(H)\big],
\qquad
\frac{\text{d}\mathbf{P}'_\pi}{\text{d}\mathbf{P}_\pi}(h)
\;=\;
\frac{1}{\alpha}\,\mathbf{1}\{\,\mathbf{J}(h)\le q_\alpha(\mathbf{J})\,\}.
\label{eq:cvar-deformation}
\end{equation}
where $q_\alpha(\mathbf{J})$ is the $\alpha$-quantile of $\mathbf{J}$ under $\mathbf{P}_\pi$. 
Thus, the non-standard aggregator has been absorbed into a reweighting of $\mathbf P_\pi$ that concentrates mass on the worst $\alpha$-fraction of outcomes. 

With respect to \condA, $\tilde{\mathbf{F}}^g$ always satisfies A2 but generally does not satisfy A1. Luckily, representation duality can recover Bellman recursion for a deformed distribution $\mathbf{P}'_\pi$. 
The transformation $\mathbf{P}_\pi\to\mathbf{P}'_\pi$ may depend on $\mathbf{J}$ --- for instance, the reweighting $\text{d}\mathbf{P}'/\text{d}\mathbf{P}_\pi$ in \autoref{eq:cvar-deformation} is a function of the rank ordering of $\mathbf{J}$. 
Thus, this is not a fixed deformation that could be applied without knowledge of $\mathbf{J}$. 

\section{Discussion}

We have shown that Bellman equations and the optimality principle are consequences of the structural compatibility of dynamics, return, and aggregation, which can be stated as three fundamental conditions. 
Many reinforcement learning frameworks can be understood as relaxing one of these conditions, and restoring it by augmenting the state or by deforming returns or probabilities. 
Our framework provides insight into how such strategies work and the limits of their applicability.

These three conditions were found to give rise to three dualities, which synthesise and extend a substantial body of prior work across reinforcement learning, optimal control, statistics, theoretical neuroscience, and category theory. 
\begin{itemize}
    \item Object duality explores symmetries induced by the expectation operator, which treats both returns and probabilities linearly. This enables exchangeability between them, underlying the utility–probability conjugacy of \cite{Jeffrey1965} and \cite{Bolker1966} and the usage of priors instead of returns in active inference \citep{Friston2015}.  
    \item Assessment duality explores how non-linearities in how returns are aggregated can be turned into either return curvature or reward regularisation, being the foundation behind the regularisation observed in control-as-inference~\citep{Levine2018}, maximum-entropy  \citep{Ziebart2010} and robust reinforcement learning \citep{husain2021regularized}.
    \item Representation duality explores how non-linearities in how probabilities are aggregated can be turned into deformations of the distribution, being closely related to Choquet~\citep{schmeidler1989subjective} and max-min expected utility~\citep{gilboa1989maxmin}.
\end{itemize}
The mathematical foundation is the categorical and coalgebraic treatment of dynamical systems \citep{FeysHansenMoss2018,HedgesRodriguezSakamoto2024RL,Fritz2020Synthetic}, which already formalises some of these dualities in isolation. For details regarding related work, see \autoref{sec:related-work}.

Thus, while each of these dualities was previously known, our contributions are: 
(i) the formulation of a general objective and conditions on dynamics, return, and aggregation such that the Bellman equation is the result of a structural compatibility, 
(ii) the treatment of state augmentation and return or dynamics deformation as a recourse when a condition fails; (iii) a synthesis identifying the three dualities as pairwise symmetries of the transition, return, and aggregation functionals, each holding one of the conditions fixed; and 
(iv) the derivation of existing formulations in various subfields as special cases of these dualities. 

\subsubsection*{Acknowledgments}
\label{sec:ack}

The authors thank Cameron Allen, Artemy Kolchinsky, Alexander Gietelink Oldenziel, and Pedro Ortega for insightful discussions and useful feedback.

\bibliography{main}
\bibliographystyle{rlj}

\beginSupplementaryMaterials

\appendix

\section{Proofs}
\label{sec:appendix1}

For notational convenience, the proofs will denote
the combined action-transition kernel as
\begin{equation}
\label{eq:def_Ktilde}
\widetilde{K}(a, s \mid z) \;:=\; \bar{\pi}(a \mid z)\;
K(s \mid z^D, a),
\qquad \widetilde{K} : \mathcal{Z} \to \Delta(\mathcal{A} \times \mathcal{S}).
\end{equation}

\subsection{Proof of \autoref{thm:generalised-bellman}}
\label{eq:theorem1}

\generalisedbellman*

\begin{proof}
To start, let us note that \condR's one-step recursion (\autoref{eq:cond-R-recursion}) implies that 
\begin{equation}
\label{eq:comp_J}
    \mathbf{J}^{(t)}(h) \;=\;
\ell\bigl(z^R_t, a_t, s_{t+1}\bigr) \;\oplus\; \mathbf{J}^{(t+1)}(h).
\end{equation}
Additionally, by \condD and
$\pi \in \Pi_\phi$, the conditional trajectory law admits the
two-stage factorisation
\begin{equation}
\label{eq:k_tilde}
\mathbf{P}_\pi(\cdot \mid h_{t}) \;=\;
\widetilde{K}(\cdot \mid z_t) \times \mathbf{P}_\pi(\cdot \mid h_{t+1})
\end{equation}
with $z_t=\phi(h_t)$, where the first stage samples $(a_t, s_{t+1})$ from $\widetilde{K}(\cdot \mid z_t^D)$
and the second samples the remaining suffix from
$\mathbf{P}_\pi(\cdot \mid h_{t+1})$ with $h_{t+1} = (h_{t},a_t, s_{t+1})$.
By \autoref{lem:value-factorisation} applied at time $t+1$,
$\mathbf{P}_\pi(\cdot \mid h_{t+1})$ depends on $h_{t+1}$ only through
$\phi(h_{t+1}) = \nu(z_t, a_t, s_{t+1})$, so the second factor is a
genuine kernel on $(\mathcal{A} \times \mathcal{S}) \to \Delta(\mathcal{T})$.

With all these pieces in place, we can prove the result as follows:\begin{align*}
&\bar V_\pi(z_t)
= \mathbf{F}_{H \sim \mathbf{P}_\pi(\cdot \mid h_{t})}\!\bigl[\,\mathbf{J}^{(t)}(H)\,\bigr]
&& \text{(def.\ + L\ref{lem:value-factorisation})}\\
&= 
\mathbf{F}_{H \sim \mathbf{P}_\pi(\cdot \mid h_{t})}\!\bigl[\,\ell(z^R_t, A_t, S_{t+1}) \oplus \mathbf{J}^{(t+1)}(H)\,\bigr]
&& \text{(\autoref{eq:comp_J})}\\
&= 
\mathbf{F}_{A_t,S_{t+1} \sim \widetilde{K}(\cdot \mid z_t)} \circ 
\mathbf{F}_{H \sim \mathbf{P}_\pi(\cdot \mid h_{t+1})}\!\bigl[\,\ell(z^R_t, A_t, S_{t+1}) \oplus \mathbf{J}^{(t+1)}(H)\,\bigr]
&& \text{(A1 + \autoref{eq:k_tilde})}\\
&= 
\mathbf{F}_{A_t,S_{t+1} \sim \widetilde{K}(\cdot \mid z_t)}\!\Bigl[\, \ell(z^R_t, A_t, S_{t+1}) \,\odot\, \mathbf{F}_{H \sim \mathbf{P}_\pi(\cdot \mid h_{t+1})}\!\bigl[\,\mathbf{J}^{(t+1)}(H)\,\bigr]\,\Bigr]
&& \text{(A2)}\\
&= 
\mathbf{F}_{A_t,S_{t+1} \sim \widetilde{K}(\cdot \mid z_t)}\!\Bigl[\, \ell(z^R_t, A_t, S_{t+1}) \,\odot\, \bar V_\pi\bigl(\nu(z_t, A_t, S_{t+1})\bigr)\,\Bigr]
&& \text{(def.\ + L\ref{lem:value-factorisation})}\\
&= 
\mathbf{F}_{A_t~\sim\bar{\pi}(\cdot|z_t)}
\circ
\mathbf{F}_{S_{t+1} \sim K(\cdot \mid z^D_t,A_t)}\!\Bigl[\, \ell(z^R_t, A_t, S_{t+1}) \,\odot\, \bar V_\pi\bigl(\nu(z_t, A_t, S_{t+1})\bigr)\,\Bigr]
&& \text{(A1 + \autoref{eq:def_Ktilde})}.
\end{align*}
Renaming the dummy variables $z \leftrightarrow z_t$ gives
\autoref{eq:generalised-bellman}.
\end{proof}

\subsection{Proof of \autoref{thm:bellman-optimality}}
\label{eq:theorem2}

\begin{figure}
    \centering
    \begin{tikzpicture}[
    >=Stealth,
    every node/.style={font=\small},
    condbox/.style={
        rounded corners=4pt,
        draw=#1, line width=1pt,
        fill=#1!8,
        text width=2.95cm,
        align=center,
        inner sep=3pt
    },
    statebox/.style={
        rounded corners=5pt,
        draw=black!72, line width=1.3pt,
        fill=bellmanfill,
        align=center,
        inner sep=5pt
    }
]

\node[condbox=condD] (condD) at (0, 1.5) {
    \textbf{Condition D}\\[1pt]
    {\scriptsize History sufficiency}\\[2pt]
};
\node[condbox=condR] (condR) at (0, 0) {
    \textbf{Condition R}\\[1pt]
    {\scriptsize Return compositionality}\\[2pt]
};
\node[condbox=condF] (condA) at (0, -1.5) {
    \textbf{Condition A}\\[1pt]
    {\scriptsize Aggregation compatibility}\\[2pt]
};

\node[statebox] (eval) at (5.5, 0) {
    \textbf{Bellman recursion}\\[3pt]
    {\footnotesize $\bar V_\pi(z)=
    (\mathbf{T}_\pi\bar V_\pi)(z)$}
\\[3pt]{\footnotesize\itshape \autoref{thm:generalised-bellman}}
};

\draw[->, condD, line width=1.2pt, shorten >=5pt, shorten <=2pt] (condD.east) -- (eval.west);
\draw[->, condR, line width=1.2pt, shorten >=5pt, shorten <=2pt] (condR.east) -- (eval.west);
\draw[->, condF, line width=1.2pt, shorten >=5pt, shorten <=2pt] (condA.east) -- (eval.west);

\node[statebox] (opt) at (11.0, 0) {
    \textbf{Bellman optimality}\\[3pt]
    {\footnotesize $\bar V^{*}(z)=(\mathbf{T}_*\bar V^*)(z)$}
    \\[3pt]{\footnotesize\itshape \autoref{thm:bellman-optimality}}
};

\draw[->, line width=1.4pt, black!78]
    (eval.east) -- node[below=2pt, font=\scriptsize, text=black!62]
    {} (opt.west);

\node[condbox=condM, text width=2.5cm, fill=condM!10] (condM) at (8.2, 1.5) {
    {$\mathbf F,\ \odot$ monotonic}
};
\draw[->, condM, line width=1.1pt, dashed, shorten >=2pt] (condM.south) -- (8.2, 0.05);

\end{tikzpicture}
    \caption{\small{Conditions giving rise to generalised Bellman recursion and optimality. The three conditions D, R, A yield a \emph{generalised Bellman recursion} (\autoref{thm:generalised-bellman}). Adding an additional monotonicity assumption on $\mathbf{F}$ and $\odot$ yields the Bellman \emph{optimality} equation (\autoref{thm:bellman-optimality}).
    }} \label{fig:conditions}
\end{figure}
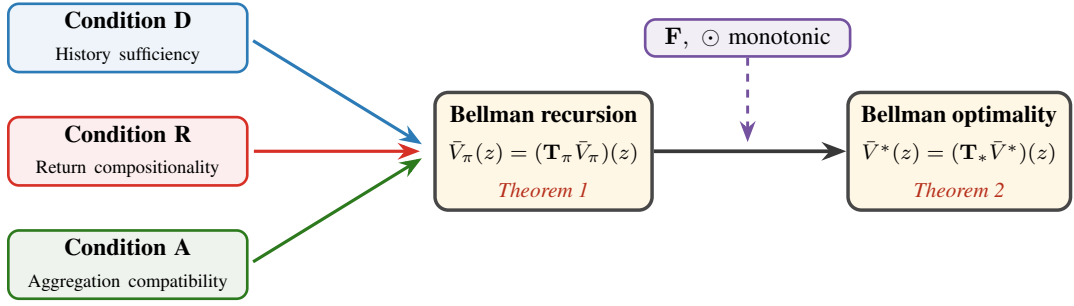

\bellmanoptimality*

\begin{proof}
By \autoref{lem:opt-factorisation}, $V^*(h_{t})=\bar V^*\big(\phi(h_{t})\big)$. We prove the
history-level optimality identity
\begin{equation}\label{eq:history-optimality}
  V^*(h_{t})
  = \sup_{\mu\in\Delta(A)}
    F_{A\sim\mu,S\sim K(\cdot|\phi_D(h_{t}),A)}
    \Big[\,\ell\big(\phi_R(h_{t}),A,S\big)\ 
    \odot\ V^*\big((h_{t},A,S)\big)\,\Big],
\end{equation}
which equals $(\mathbf{T}_*\bar V^*)(z)$ after factorisation.
For notational convenience, the proof will use
\begin{equation}
\widetilde{K}_\pi(a,s|h_{t}) := \pi(a\mid h_{t}) K(s\mid \phi_D(h_{t}),a).
\end{equation}

\medskip
\emph{Upper bound.} Let $\pi\in\Pi$ and write $\mu_\pi:=\pi(\cdot\mid h_{t})$.
Since $V_\pi\le V^*$ pointwise, monotonicity of $\odot$ and of $\mathbf{F}$ applied to \autoref{prop:history-recursion} give
\[
  V_\pi(h_{t})
  = F_{\widetilde{K}_{\mu_\pi}}\!\big[\ell\odot V_\pi(\cdot)\big]
  \le F_{\widetilde{K}_{\mu_\pi}}\!\big[\ell\odot V^*(\cdot)\big]
  \le \sup_{\mu\in\Delta(A)} F_{\widetilde{K}_{\mu}}\!\big[\ell\odot V^*(\cdot)\big].
\]
Taking the supremum over $\pi\in\Pi$ on the left yields
$V^*(h_{t})\le \text{RHS of \autoref{eq:history-optimality}}$.

\medskip
\emph{Lower bound ($\ge$).} Fix $\mu\in\Delta(A)$. For each successor $(h_{t},a,s)$ choose an
attaining policy $\pi_{a,s}\in\Pi$ with $V_{\pi_{a,s}}\big((h_{t},a,s)\big)=V^*\big((h_{t},a,s)\big)$.
Define $\pi^\mu\in\Pi$ to play $\mu$ at $h_{t}$ and, after observing $(a,s)$, to follow
$\pi_{a,s}$. Because $\Pi$ is the full class of history-dependent policies, it is closed under
this one-step concatenation, so $\pi^\mu\in\Pi$. Applying
\autoref{prop:history-recursion} to $\pi^\mu$ gives\[
  V_{\pi^\mu}(h_{t})
  = F_{\widetilde{K}_\mu}\!\big[\ell\odot V_{\pi^\mu}(\cdot)\big]
  = F_{\widetilde{K}_\mu}\!\big[\ell\odot V^*(\cdot)\big],
\]
the second equality holding because $V_{\pi^\mu}\big((h_{t},a,s)\big)=V_{\pi_{a,s}}\big((h_{t},a,s)\big)=V^*\big((h_{t},a,s)\big)$.
Since $\pi^\mu\in\Pi$ we have $V_{\pi^\mu}(h_{t})\le V^*(h_{t})$, hence
$F_{\widetilde{K}_\mu}[\ell\odot V^*(\cdot)]\le V^*(h_{t})$. Taking the supremum over $\mu\in\Delta(A)$
gives $V^*(h_{t})\geq \text{RHS}$ of \autoref{eq:history-optimality}.

\medskip
The two bounds establish \autoref{eq:history-optimality}, and factorising through $\phi$, gives
$\bar V^*=\mathbf{T}_*\bar V^*$.
\end{proof}

\subsection{Other results}
\label{sec:appendix_proof_lemma}

Here we state and prove auxiliary results that are used to prove the main theorems.

\begin{lemma}[Value factorisation]
\label{lem:value-factorisation}
If \condD and \condR hold, then
\begin{equation}
\label{eq:value-factorisation}
\phi(h_{t}) = \phi(h'_{t}) \;\Longrightarrow\;
V_\pi(h_{t}) = V_\pi(h'_{t}),
\quad 
\forall \pi \in \Pi_\phi, t \ge 0.
\end{equation}
\end{lemma}

\begin{proof}
By \autoref{def:history-value}, $V_\pi(h_{t})$ is determined by the joint distribution of $\mathbf{J}^{(t)}(H)$ when $H \sim \mathbf{P}_\pi(\cdot \mid h_{t})$. It therefore suffices to show that, as $h_{t}$ varies over histories with a fixed value of $\phi(h_{t})$, both the conditional law $\mathbf{P}_\pi(\cdot \mid h_{t})$ (Step 1) and the function $H \mapsto \mathbf{J}^{(t)}(H)$ (Step 2) are unchanged.

\emph{Step 1.} By \condD and $\pi \in \Pi_\phi$, the
conditional trajectory distribution factors as
\begin{equation}
\label{eq:cond-law-factorisation}
\mathbf{P}_\pi(h \mid h_{t}) \;=\; \prod_{\tau \ge t}
\bar{\pi}\bigl(a_{\tau} \mid \phi(h_{\tau})\bigr)\;
K\bigl(s_{\tau+1} \mid \phi_D(h_{\tau}),\, a_{\tau}\bigr).
\end{equation}
Due to the unifilarity of $\phi=(\phi_D,\phi_R)$, every $\phi(h_{\tau})$ with $\tau \ge t$ is determined by $\phi(h_{t})$
and the intermediate pairs $(a_t, s_{t+1}), \dots, (a_{\tau-1}, s_{\tau})$.
Hence, \autoref{eq:cond-law-factorisation} depends on $h_{t}$
only through $\phi(h_{t})$.

\emph{Step 2.} By \condR,
\begin{equation*}
\mathbf{J}^{(t)}(h) \;=\; \bigoplus_{\tau \ge t}
\ell\bigl(\phi_R(h_{\tau}),\, a_{\tau},\, s_{\tau+1}\bigr).
\end{equation*}
By the unifilarity of $\phi_R$, 
each $\phi_R(h_{\tau})$ for $\tau \ge t$ is determined by
$\phi_R(h_{t})$ and the same intermediate pairs. Hence the local
returns $\ell(\phi_R(h_{\tau}), a_{\tau}, s_{\tau+1})$ depend on
$h_{t}$ only through $\phi_R(h_{t})$.

Finally, the functional $V_\pi(h_{t})$ is determined by the distribution of $\mathbf{J}^{(t)}(h)$ under $H \sim \mathbf{P}_\pi(\cdot \mid h_{t})$.
Steps 1 and 2 show that both the law and the random variable depend on
$h_{t}$ only through $\phi(h_{t})$, so the distribution of
$\mathbf{J}^{(t)}(H)$ does, and hence so does $V_\pi(h_{t})$.
\end{proof}

\begin{lemma}[Factorisation of the optimal value]\label{lem:opt-factorisation}
Assume \condD and \condR. Define $V^*(h_{t}) := \sup_{\pi\in\Pi} V_\pi(h_{t})$.
Then for any two histories with $\phi(h_{t})=\phi(h'_{t})$,
\[
  V^*(h_{t}) = V^*(h'_{t}),
\]
so $V^*$ factors through $\phi$ as $V^*(h_{t}) = \bar V^*\big(\phi(h_{t})\big)$ for a unique
$\bar V^*:\mathcal Z\to\mathbb{R}$.
\end{lemma}

\begin{proof}
Take $h_{t},h'_{t}\in\mathcal{H}$ with $\phi(h_{t})=\phi(h'_{t})=z$. 
For a given suffix $u=(a_t,s_{t+1},a_{t+1},s_{t+2},\dots)$ and each $\tau\ge t$, write
\[
  h^{u}_{\tau} := (h_{t},a_t,s_{t+1},\dots,a_{\tau-1},s_\tau),
  \qquad
  (h')^{u}_{\tau} := (h'_{t},a_t,s_{t+1},\dots,a_{\tau-1},s_\tau)
\]
for the two histories obtained by appending the same suffix to $h_{t}$ and to $h'_{t}$. 

\medskip
\emph{Step 1.} We prove by induction that
\begin{equation}\label{eq:suffix-match}
  \phi\big(h^{u}_{\tau}\big) = \phi\big((h')^{u}_{\tau}\big)
  \qquad \text{for all } \tau\ge t .
\end{equation}
At $\tau=t$ (taking the empty suffix) the claim reduces to the
hypothesis $\phi(h_{t})=\phi(h'_{t})$. Now, suppose \autoref{eq:suffix-match} holds at $\tau$.
By the unifilarity of $\phi=(\phi_D,\phi_R)$ guaranteed by \condD and \condR, there is a map
$\nu:\mathcal Z\times \mathcal A\times \mathcal S\to \mathcal Z$ with $\phi(h_{\tau+1})=\nu\big(\phi(h_{\tau}),a_{\tau},s_{\tau+1}\big)$ for every $h_{\tau+1}$.
Applying this to both extended histories and using the inductive hypothesis, one finds that
\[
  \phi\big(h^{u}_{\tau+1}\big)
  = \nu\big(\phi(h^{u}_{\tau}),a_{\tau},s_{\tau+1}\big)
  = \nu\big(\phi((h')^{u}_{\tau}),a_{\tau},s_{\tau+1}\big)
  = \phi\big((h')^{u}_{\tau+1}\big).
\]
This implies that $\phi_D$ and $\phi_R$ agree along the two
suffix-extended histories at every $\tau\ge t$.

\medskip
\emph{Step 2.} The value $V_\pi(h_{t})$ depends on $\pi$ only through its action distributions on histories
extending $h_{t}$. Define a relabelling $\Theta$ of these restrictions by mimicking along
matched suffixes: for every finite suffix $u$,
\[
  \Theta(\pi)\big((h')^{u}_{\tau}\big) := \pi\big(h^{u}_{\tau}\big),
  \qquad \tau\ge t .
\]
We show $V_\pi(h_{t})=V_{\Theta(\pi)}(h'_{t})$ by checking that the trajectory law and the
tail return agree.

By \condD, the conditional trajectory law factorises along the suffix
through $\phi_D$:
\[
  P_\pi\big(u\mid h_{t}\big)
  = \prod_{\tau\ge t} \pi\big(a_\tau\mid h^{u}_{\tau}\big)\,
    K\big(s_{\tau+1}\mid \phi_D(h^{u}_{\tau}),a_\tau\big).
\]
By Step~1, $\phi_D(h^{u}_{\tau})=\phi_D((h')^{u}_{\tau})$ for every $\tau$, and by construction
$\Theta(\pi)\big((h')^{u}_{\tau}\big)=\pi\big(h^{u}_{\tau}\big)$, hence each factor is unchanged,
giving $P_{\Theta(\pi)}(\,\cdot\mid h'_{t})=P_\pi(\,\cdot\mid h_{t})$.

\medskip
\emph{Step 3.} By \condR, the tail return decomposes as
\[
  \mathbf J^{(t)}\big(h_{t},u\big)
  = \bigoplus_{\tau\ge t} \ell\big(\phi_R(h^{u}_{\tau}),a_\tau,s_{\tau+1}\big),
\]
and likewise from $h'_{t}$. By Step~1, $\phi_R(h^{u}_{\tau})=\phi_R((h')^{u}_{\tau})$ for every
$\tau$, so the two $\oplus$-sums have identical terms:
$\mathbf J^{(t)}\big(h_{t},u\big)=\mathbf J^{(t)}\big(h'_{t},u\big)$ for every suffix $u$. 

Since the law of the suffix and the tail-return functional both coincide, the random variable
$\mathbf J^{(t)}$ has the same distribution under $(\pi,h_{t})$ and under $(\Theta(\pi),h'_{t})$.
By \autoref{def:history-value}, $V_\pi(h_{t})=V_{\Theta(\pi)}(h'_{t})$.

\medskip
\emph{Step 4: equality of suprema.}
The map $\Theta$ is a bijection between the policy restrictions on histories extending $h_{t}$
and those extending $h'_{t}$, its inverse mimicking in the opposite direction. Hence the sets of
achievable values $\{V_\pi(h_{t}):\pi\in\Pi\}$ and $\{V_{\pi'}(h'_{t}):\pi'\in\Pi\}$ coincide,
and therefore so do their suprema:
\[
  V^*(h_{t}) = \sup_{\pi\in\Pi} V_\pi(h_{t})
             = \sup_{\pi'\in\Pi} V_{\pi'}(h'_{t})
             = V^*(h'_{t}).
\]
As $V^*$ takes a common value on every fibre $\{\,h_{t}: \phi(h_{t})=z\,\}$, it factors through
$\phi$, defining a unique $\bar V^*:Z\to\mathbb{R}$ with $V^*(h_{t})=\bar V^*(\phi(h_{t}))$.
\end{proof}

\begin{lemma}[History-level Bellman recursion]\label{prop:history-recursion}
Assume \condD, \condR, and \condA hold. For every $\pi\in\Pi$ and
every $h_{t}\in\mathcal{H}$,
\begin{equation}\label{eq:history-recursion}
  V_\pi(h_{t})
  \;=\;
  F_{A\sim \pi(\cdot\mid h_{t}),S\sim K(\cdot\mid \phi_D(h_{t}),A)}
  \Big[\,\ell\big(\phi_R(h_{t}),A,S\big)\ \odot\ V_\pi\big((h_{t},A,S)\big)\,\Big].
\end{equation}
\end{lemma}

\begin{proof}
The argument is identical to Steps~1--3 of the proof of \autoref{thm:generalised-bellman} but with $\bar{\pi}(a\mid z)$ replaced by $\pi(a\mid h_{t})$. Those steps
only use Conditions D, R, and A. None of them require $\pi\in\Pi_\phi$ (that restriction
is invoked only afterwards, through \autoref{lem:value-factorisation}, to re-express
$V_\pi$ as a function of $z=\phi(h_{t})$). 
\end{proof}

\section{On the role of \condD}
\label{app:condition-d}

Here we show that \condD is
formally trivial if arbitrary history-valued statistics are allowed. 
However, this does not make its role vacuous as part of Bellman's construction. It serves four
purposes.

\begin{enumerate}
    
\item It captures the compression problem, i.e., whether a proposed state $z_t$ 
contains all past information relevant to accounting for future uncertainty. 
\item It separates dynamics-sufficiency (via $\phi_D$) from
return-sufficiency (via $\phi_R$). Even if the full history satisfies both, a useful
Bellman equation requires a common compressed state. 
\item It clarifies where non-stationarity is handled: time must either be irrelevant or be
included in the statistic. 
\item It allows us to generalise the framework from probability theory to other uncertainty formalisms (see App.~\ref{sec:beyond_probability}). 
\end{enumerate}

Thus, the core of \condD is not that some $\phi_D$ exists, but is that the chosen uncertainty representation admits a local, recursively updateable, and hopefully compressed statistic that is compatible with the return and aggregation structures used by the Bellman recursion.

\subsection{Clarifications in probabilistic settings}

\condD is stated as the existence of a unifilar statistic
$\phi_D:\mathcal H\to\mathcal Z_D$ satisfying
\begin{equation}
\label{eq:cond-D1_app}
\mathbf{P}_\pi(h_t) 
= 
\rho(s_0)
\prod_{\tau = 0}^{t-1}
\pi(a_\tau \,\big|\, h_{\tau})\;
K\bigl(s_{\tau+1} \,\big|\, \phi_D(h_{\tau}),a_\tau\bigr).
\end{equation}
Here, we consider some examples. 

\emph{History-dependent dynamics.} For an arbitrary process, the
trivial statistic $\phi_D(h_{t})=h_{t}$ with $\mathcal Z_D=\mathcal H$ satisfies \condD, but gives no
compression and no tractable Bellman equation in general. The
recursion is then merely a recursion on the tree of histories, since
$h_{t+1}=(h_{t},a_t,s_{t+1})$. 
In other words, every stochastic process is Markov when the state is taken to be its complete past. 

\emph{Markovian dynamics.} If \condD is satisfied with $\phi_D(h_{t})=s_{t}$, then the scenario corresponds to a Markov decision process. 
If the transition dynamics are time-inhomogeneous, 
then $\phi_D(h_{t})=s_{t}$ is generally not sufficient, but
$\phi_D(h_{t})=(t,s_{t})$ is. Thus, a time-inhomogeneous model can
still satisfy \condD after augmenting the state by adding a clock, so $\mathcal{Z}_D=\mathcal{S}\times\mathbb{N}$.

\emph{Partially-observable Markov decision processes (POMDPs).} In a POMDP, the state $s_t$ is not directly observed. Instead, observations are obtained from a probabilistic mapping $\mathcal{S} \to \Delta(\Omega)$, where $\Omega$ is a set of possible observations for the agent. However, these observations are not dynamically-sufficient in general. POMDPs with known dynamics can be converted into a \emph{belief-space MDP}, where each state is a probability distribution representing the agent's belief about the hidden state $s \in \mathcal{S}$. The belief state
$b_t\in\Delta(\mathcal S)$, updated by Bayes' rule from the previous
belief, the action, and the observation, is the canonical
dynamics-sufficient statistic. In this case, \condD is the statement
that the belief process is Markovian, even though the observation process may not be.

\subsection{Beyond probability}
\label{sec:beyond_probability}

The same distinction between trivial full-history representations and
a useful compressed representation can be generalised to more general notions of 
uncertainty. Let $\mathsf M(Y)$ denote the space of uncertainty objects over a space $Y$. An \emph{uncertainty kernel} is a map
\[
L:X\to \mathsf M(Y).
\]
When $\mathsf M=\Delta$, these are ordinary Markov kernels. More
generally, $\mathsf M$ may be a powerset functor for nondeterministic
systems, a semiring-valued distribution functor for weighted systems,
or a credal set for imprecise probabilities \citep{uncertainty2024}.

In such a setting, \condD should be understood as requiring local
unrolling of future uncertainty through the statistic $\phi_D$. That
is, for each policy $\pi$, there is a one-step uncertainty kernel
\[
\Lambda_\pi:\mathcal Z_D\to \mathsf M(\mathcal A\times\mathcal S)
\]
such that the uncertainty object over future suffixes from $h_{t}$
is obtained by first drawing, selecting, or weighting one local
action-state pair according to $\Lambda_\pi(\phi_D(h_{t}))$, and
then continuing from the updated history
$(h_{t},a_t,s_{t+1})$. Abstractly, if the uncertainty calculus has a
sequential composition operation, written here as
$\operatorname{bind}$, this takes the schematic form
\[
\mathbf{P}_\pi^t(\cdot\mid h_{t})
=
\Lambda_\pi(\phi_D(h_{t}))
\operatorname{bind}
\Bigl((a_t,s_{t+1})\mapsto
\mathbf{P}_\pi^{t+1}(\cdot\mid h_{t},a_t,s_{t+1})\Bigr),
\]
with the unifilar update
\[
\phi(h_{t+1})=
\nu\big(\phi(h_{t}),a_t,s_{t+1}\big).
\]
For ordinary probability, this is the usual one-step decomposition of
the conditional trajectory law. For other uncertainty calculi, it says
that future uncertainty can still be generated recursively from a
state-dependent local uncertainty object.

\paragraph{Nondeterministic uncertainty.}

Let $\mathsf M(X)=\mathsf P(X)$ be the powerset of possible outcomes.
An uncertainty kernel is then a set-valued map $L:X\to\mathsf P(Y)$.
If
\[
\bar S(z,a)\subseteq \mathcal S
\]
is the set of possible next states from statistic $z$ and action $a$,
then \condD says that the possible future suffixes from $h_{t}$ are
generated by choosing $s_{t+1}\in \bar S(\phi_D(h_{t}),a_t)$ and then
continuing from the updated statistic. The full history again gives a
trivial representation, and the useful question is whether the set of
possible futures depends only on a smaller state.

\paragraph{Imprecise and robust uncertainty.}

For robust or imprecise models, $\mathsf M(X)$ may be a set of
probability measures on $X$. A one-step model may specify an ambiguity
set
\[
\bar{\mathcal K}(z,a)\subseteq \Delta(\mathcal S)
\]
rather than a single transition kernel. A dynamics statistic is
sufficient when this ambiguity set, and the way it is updated through
time, depends on the past only through $\phi_D(h_{t})$. However,
Bellman recursions for such models generally require more than this:
the ambiguity sets must compose in a dynamically consistent way. In
robust MDPs this is known as the rectangularity requirement~\citep{Wiesemann2013}. In the
present framework, that additional requirement belongs to the analogue
of \condA: the aggregation $\mathbf F$ must respect the chosen
sequential composition of uncertainty objects.

\section{Active inference and the free-energy principle}
\label{app:active_inference}

This appendix develops the connection between \autoref{eq:MDL} and the active inference framework~\citep{Friston2015,DaCosta2020,Sajid2021}. 

Active inference casts the symmetry between beliefs and preferences as constitutive of agency: preferences are represented as prior preferences encoded in a generative model, and the agent's task is to minimise variational and expected free energy across perception and action, respectively. The framework has been developed extensively as a descriptive model of biological agency \citep{Friston2010}, with sustained engagement and critique \citep{biehl2021technical,Sajid2021,Millidge2024Retrospective}.

Additionally, there exist interesting relationships between active inference and control-as-inference.  \citet{Millidge2020AIFvsCAI} have shown that the two frameworks differ primarily in how reward enters the generative model --- i.e., priors over outcomes in active inference, auxiliary optimality variables in control-as-inference, and that the underlying inferential structures are related by the process of using inference procedures to compute policies. In the following, we explain how the distinction between active inference vs control-as-inference introduced by \cite{Millidge2020AIFvsCAI} corresponds, in our framework, to a choice of \emph{how} the object duality is implemented.\footnote{We view our framework as neutral between the active-inference and standard-RL philosophies. That said, the broader programmatic claims of active inference (that the free energy principle is a normative theory of biological agency, that all of perception and action reduces to variational inference) lie outside the scope of our framework, being claims about the empirical scope of the object duality, not about its formal content.}

\subsection{Object duality in active inference}
\label{app:vfe:setup}
For a finite outcome or trajectory space
and $\beta>0$, define
\begin{equation}
\tilde P_\beta(h)
:=
\frac{\exp[\beta\mathbf J(h)]}{Z_\beta},
\qquad
Z_\beta:=\sum_h\exp[\beta\mathbf J(h)].
\label{eq:aif-gibbs-preference}
\end{equation}
Then, for any predictive trajectory law $Q_\pi$,
\begin{equation}
-\mathbb E_{Q_\pi}\log\tilde P_\beta(H)
=
-\beta\,\mathbb E_{Q_\pi}[\mathbf J(H)]+\log Z_\beta.
\label{eq:aif-pragmatic}
\end{equation}
\autoref{eq:aif-pragmatic} is the formal expression of the
active-inference identification of preferences with prior probabilities: the
agent treats high-return trajectories as a priori more likely under a
generative model, whose observable marginal is the Gibbs measure $q$.
This is the `biased generative model' construction discussed
by~\cite{Friston2015} and analysed critically
by~\cite{millidge2021whence}.
Thus, prior preferences encode the original return up to a positive scale and a policy-independent constant. This is the part of active inference that is captured by object duality. Moreover, maximising the expected return is equivalent to minimising the
cross-entropy 
\begin{equation}
    H(Q_\pi; \tilde P_\beta) = -\mathbb{E}_{H\sim Q_\pi}\big[\log \tilde P_\beta(H)\big], 
\end{equation}
since $Z$ does not depend on $\pi$. 

\subsection{Variational free energy for perceptual inference}

Let us now expand the setting to consider a partially observable structure, which distinguishes the latent state $s \in \mathcal S$ from \textit{observations} $o \in \mathcal{O}$, where $\mathcal{O}$ is a set of possible observations.
We assume that the agent does not necessarily have access to the realisation of the latent states directly throughout, but only sees the observations which are sampled from a likelihood function $P(o \mid s)$. In active inference, the agent's actions only affect the state directly through a Markovian transition kernel $P(s \mid a)$, and only affects the agent's observations indirectly through the effect of their actions on the hidden state. This is a consequence of the \textit{Markov blanket} assumption, which defines active and sensory states ($a$ and $o$ respectively) as a statistical boundary between the agent and its environment. The agent is also assumed to possess \textit{internal states}, which allow it to perform perceptual and active inference. Perceptual inference is modelled as variational inference, where the agent maintains an approximate posterior $Q(s\mid o)$ that aims to approximate the true Bayesian posterior $P(s \mid o)$. The \textit{variational free energy} (VFE) is given by
\begin{equation}
 F(o;Q)
 :=
 \mathbb E_{Q(s\mid o)}
 \left[
   \log Q(S\mid o)-\log P(o,S)
 \right].
 \label{eq:aif-vfe}
\end{equation}
Factorising $P(o,s)$ into the product of a likelihood and marginal state distribution gives the standard
complexity--accuracy decomposition
\begin{equation}
 F(o;Q) = \underbrace{D_{\mathrm{KL}}\!\left(Q(S\mid o)\,\|\,P(S)\right)}_{\text{complexity}} - \underbrace{\mathbb E_{Q(S\mid o)}\log P(o\mid S)}_{\text{accuracy}}.
 \label{eq:aif-vfe-accuracy-complexity}
\end{equation}
Alternatively, using $P(o,s)=P(s\mid o)P(o)$ gives
\begin{equation}
 F(o;Q) = \underbrace{-\log P(o)}_{\text{surprisal}} + \underbrace{D_{\mathrm{KL}}\!\left(Q(S\mid o)\,\|\,P(S\mid o)\right)}_{\text{Divergence}}.
 \label{eq:aif-vfe-posterior-gap}
\end{equation}
Variational free energy therefore upper-bounds the surprise
$-\log P(o)$, with equality when the approximate posterior is exact. Finally, the VFE can be rearranged as the sum of an energy and entropy term, which can be related to the Helmholtz free energy:
\begin{equation}
F(o;Q) = -\underbrace{\mathbb{E}_{Q(S\mid o)}\log P(o,S)}_{\text{Energy}} - \underbrace{H[Q(S|o)]}_{\text{Entropy}}.
\end{equation}

\subsection{Expected free energy for policy selection}
\label{app:aif:efe}

Bringing preferences back into the picture, consider a preference model $\tilde{P}_\beta(O)$ over observations. The \textit{expected free energy} (EFE) of a given open-loop policy is defined as~\citep{parr2022active} 
\begin{align}
 G(\pi)
 =&
 -\underbrace{\mathbb E_{Q_\pi(O)}\big[\log \tilde P_\beta(O)\big]}
 _{\text{pragmatic value}} -
 \underbrace{\mathbb E_{Q_\pi(O)}
 D_{\mathrm{KL}}\!\big(
 Q_\pi(S\mid O)\,\|\,Q_\pi(S)
 \big)}_{\text{epistemic value}}.
 \label{eq:aif-efe-pragmatic-epistemic}
\end{align}
The pragmatic value, as we have seen in \autoref{eq:aif-pragmatic}, corresponds to the expected return, whereas the epistemic term favours policies expected to produce observations that are informative about the hidden state.

\subsection{Sophisticated inference and Bellman recursion.}
Standard finite-horizon active-inference schemes typically evaluate open-loop sequences of actions. Because future actions are then not conditioned on future outcomes or beliefs, this scheme need not satisfy Bellman's principle beyond a one-step horizon as shown in~\citep{dacosta2023reward}. Sophisticated inference instead evaluates expected free energy recursively over counterfactual future beliefs, implementing a tree search over belief states using an open-loop policy
\citep{friston2021sophisticated}. In the terminology of this paper, a recursively updated belief supplies a dynamics-sufficient statistic, while recursive evaluation of stagewise expected free energy supplies the time-consistent continuation structure needed for Bellman recursion.

For finite-horizon MDPs with known transitions and state preferences $\tilde P_\beta(s)\propto\exp\{\beta R(s)\}$, \cite{dacosta2023reward} show that, in the limit $\beta\to\infty$, the standard expected free energy is Bellman optimal at horizon one but not in general at longer horizons, whereas their recursive sophisticated scheme yields a Bellman-optimal state--action policy on any finite horizon. Their extension to the POMDP setting shows that under the same assumptions and in the limit as $\beta\to\infty$, EFE-minimising open-loop policies are also reward-maximising policies.

\section{Derivations of results in assessment duality}
\label{app:first_duality_misc}

This appendix summarises existing results that prove that each of the three families of non-standard Bellman equations recovered in \autoref{subsec:return-aggregation-duality} 
--- robust (\autoref{eq:robust_bellman_equation}), risk-sensitive (\autoref{eq:risk-sensitive}), and soft (\autoref{eq:soft-bellman-standard}) --- 
can be re-expressed as a standard Bellman recursion
with regularised returns. The proofs share a common structure, 
which is set in App.~\ref{app:unified-regularised} and then applied 
in Apps.~\ref{app:proof-soft}–\ref{app:proof-robust}. 

\subsection{A unified perspective: regularised returns and the Gibbs identity}
\label{app:unified-regularised}

The technical background of the proofs is the Donsker–Varadhan / Gibbs variational
identity \citep{ortega2013thermodynamics,geist2019theory,brekelmans2022your}. For any reference probability measure $\mu \in \Delta(\mathcal{X})$,
any measurable $f : \mathcal{X} \to \mathbb{R}$ with $\mathbb{E}_\mu[\exp(\beta f)] < \infty$, and any $\beta > 0$, then the following holds:
\begin{equation}
  \frac{1}{\beta} \log \mathbb{E}_{X \sim \mu}\!\left[\exp(\beta f(X))\right]
  \;=\; \sup_{\nu \in \Delta(\mathcal{X})}
        \left\{ \mathbb{E}_{X \sim \nu}\big[f(X)\big] - \frac{1}{\beta}\, D_{\mathrm{KL}}\!\left(\nu \,\|\, \mu\right) \right\},
  \label{eq:gibbs}
\end{equation}
with the supremum attained at the `Gibbs tilt'
\begin{equation}
  \frac{d\nu^{\star}}{d\mu}(x)
  \;=\; \frac{\exp\big(\beta f(x)\big)}{\mathbb{E}_\mu\big[\exp\big(\beta f(X)\big)\big]}.
  \label{eq:gibbs-optimal}
\end{equation}
For $\beta < 0$ the identity holds with the supremum replaced by an
infimum and the same attaining tilt, and the KL coefficient
$-1/\beta = 1/|\beta|$ remains non-negative. In the limit
$\beta \to -\infty$ the KL term in the variational form is dominated
by the expectation and the right-hand side of \autoref{eq:gibbs}
converges to a hard infimum of $f$ over the support of $\mu$. 
This is the bridge between the risk-sensitive and robust
formulations.

All three proofs that follow share the same template:
\begin{enumerate}
  \item Identify the non-linear aggregation appearing in the target
        Bellman equation.
  \item Apply \autoref{eq:gibbs} (or its $\beta\to-\infty$ limit) to
        rewrite that aggregation as a sup/inf of an ordinary
        expectation under an auxiliary measure $Q$, minus a divergence.
  \item Absorb the divergence into the local return, producing a
        conjugate $\ell$ that depends on $Q$.
  \item Verify that the resulting recursion is the standard Bellman
        recursion with $\mathbf{F}=\mathbb{E}$, $\oplus=\odot=+$, and
        conjugate return $\ell$.
\end{enumerate}
The auxiliary measure $Q$ plays a different role in each case — a
policy in the soft case, a transition kernel in the risk-sensitive
case, and a transition kernel constrained to an ambiguity set in the
robust case — but the structural derivation is identical. 
In the
language of the $(\beta_a,\beta_s)$ terms as used in \autoref{eq:two-param-bellman}, 
the various formulations result from applying \autoref{eq:gibbs} at either the
\emph{state} site (risk-sensitive: $\beta_a=+\infty$, $\beta_s$ finite),
the \emph{action} site (soft: $\beta_a$ finite, $\beta_s=0$), 
or in the limit at both sites (robust: $\beta_a=+\infty$, $\beta_s=-\infty$). 

Details of the calculations and corresponding references are provided below. 

\subsection{Proof of \autoref{eq:soft-bellman-standard} --- the soft Bellman equation}
\label{app:proof-soft}

The link between the soft Bellman equation and reward regularisation has been investigated by \cite{husain2021regularized} and \cite{brekelmans2022your} (see also \citep{Ziebart2010,Haarnoja2017,Levine2018,geist2019theory,nachum2017bridging,schulman2017equivalence}).

\begin{proof}
Let $\pi_0 \in \Pi_\phi$ be a reference policy with
$\pi_0(a \mid s) > 0$ wherever any candidate policy assigns positive
probability. The soft Bellman equation (\autoref{eq:soft-bellman-standard}) can be re-stated as 
\begin{equation}
  V^{\ast}_\text{soft}(s)
  \;=\; \frac{1}{\beta_a}
        \log \sum_{a \in \mathcal{A}} \pi_0(a \mid s)
        \exp\big({\beta_a Q^{\ast}_\text{soft}(s,a)}\big), 
  \label{eq:soft-recap}
\end{equation}
where we are using the soft Q-function
\begin{equation}
  Q^{\ast}_\text{soft}(s,a)
  \;:=\; r(s,a) + \gamma\, \mathbb{E}_{S'\sim K(\cdot \mid s,a)}\big[ V^{\ast}_\text{soft}(S')\big].
  \label{eq:soft-Q}
\end{equation}
Applying the Gibbs identity (\autoref{eq:gibbs}) at each $s$ with
$\mu = \pi_0(\cdot \mid s)$, $f = Q^{\ast}(s, \cdot)$, and
$\beta = \beta_a > 0$, gives the following identity:
\begin{equation}
  \frac{1}{\beta_a}\log \mathbb{E}_{A\sim \pi_0(\cdot \mid s)}\!\Big[\exp(\beta_a\, Q^{\ast}_\text{soft}(s,A))\Big]
  = 
  \!\!\sup_{\mu \in \Delta(\mathcal{A})}
  \!\!
        \left\{ \mathbb{E}_{A \sim \mu}\big[Q^{\ast}_\text{soft}(s,A)\big] - \frac{1}{\beta_a}\, D_{\mathrm{KL}}\!\big(\mu \,\|\, \pi_0(\cdot \mid s)\big) \right\}.
  \label{eq:soft-gibbs}
\end{equation}
Combining \autoref{eq:soft-recap} with \autoref{eq:soft-gibbs}, and
expanding 
$D_{\mathrm{KL}}(\mu\,\|\,\pi_0) = \mathbb{E}_{A\sim\mu}\!\big[\log(\mu(A)/\pi_0(A|s))\big]$, leads to
\begin{align}
  V^{\ast}(s)
  &= \sup_{\mu \in \Delta(\mathcal{A})}
     \mathbb{E}_{A \sim \mu}
     \!\left[\, r(s,A) + \gamma\, \mathbb{E}_{S' \sim K(\cdot|s,A)}\big[ V^{\ast}(S')\big]
                - \frac{1}{\beta_a}\log\frac{\mu(A)}{\pi_0(A | s)} \right] \nonumber \\
  &= \sup_{\mu \in \Delta(\mathcal{A})}
     \mathbb{E}_{A \sim \mu,\, S' \sim K(\cdot \mid s, A)}
     \!\Big[\, \ell_{\text{soft}}(s, A, S'; \mu) + \gamma\,  V^{\ast}(S') \,\Big],
  \label{eq:soft-bellman-conjugate}
\end{align}
where the \emph{soft local return} is
\begin{equation}  
    \ell_{\text{soft}}(s, a; \mu)
    \;=\; r(s, a) \;-\; \frac{1}{\beta_a}\,
          \log \frac{\mu(a \mid s)}{\pi_0(a \mid s)}.
  \label{eq:ell-soft-corrected}
\end{equation}
\autoref{eq:soft-bellman-conjugate} is a standard Bellman
optimality equation with aggregation $\mathbb{E}$, composition $+$,
transition kernel $K(\cdot|s,a)$, and local return $\ell_{\text{soft}}$. 
From \autoref{eq:gibbs-optimal}, the maximising policy can be seen to be the Gibbs tilt
\begin{equation}
  \mu^{\ast}(a \mid s)
  \;=\; \frac{\pi_0(a \mid s)\,\exp(\beta_a\, Q^{\ast}_\text{soft}(s,a))}
             {\sum_{a' \in \mathcal{A}} \pi_0(a' \mid s)\,\exp(\beta_a\, Q_\text{soft}^{\ast}(s,a'))}.
  \label{eq:soft-optimal-policy}
\end{equation}
\end{proof}

An equivalent derivation, used elsewhere in the literature
\citep{Haarnoja2017,Levine2018}, starts from
the trajectory-level conjugate return (assuming $\gamma=1$)
\begin{equation}
  \mathbf{J}^{\pi}_{\mathrm{conj}}(h) \;:=\; \mathbf{J}(h) \;-\; \frac{1}{\beta_a}\,
  \log \frac{\text{d}\mathbf{P}_\pi}{\text{d}\mathbf{P}_0}(h),
  \label{eq:conj-traj}
\end{equation}
so that $\mathbb{E}_{h\sim \mathbf{P}_\pi}[\mathbf{J}^{\pi}_{\mathrm{conj}}(h)]
= \mathbb{E}_{h\sim \mathbf{P}_\pi}[\mathbf{J}(h)] - \beta_a^{-1} D_{\mathrm{KL}}(\mathbf{P}_\pi\,\|\,\mathbf{P}_0)$.
Under \condD with a shared transition kernel, the dynamics terms cancel in the density ratio, and for $\pi \in \Pi_\phi$ one has
\begin{equation}
\log
\frac{\text{d}\mathbf{P}_\pi}{\text{d}\mathbf{P}_0}(h)
= 
\sum_{\tau\geq 0} 
\log
\frac{\pi\big(a_\tau\mid\phi(h_{\tau})\big)}{\pi_0\big(a_\tau\mid\phi(h_{\tau})\big)}
\,.    
\end{equation}
Thus, $\mathbf J^{\pi}_{\mathrm{conj}}$ satisfies \condR with the local return as given in \autoref{eq:ell-soft-corrected}; substituting this into
$\mathbf{T}_*$ (\autoref{eq:bellman-operator}) and applying \autoref{eq:gibbs} recovers
\autoref{eq:soft-recap}. The forward and reverse routes therefore close
the loop between the soft Bellman equation and its conjugate-return
representation. 

\subsection{Proof of \autoref{eq:risk-sensitive} --- the risk-sensitive Bellman equation}
\label{app:proof-risk}

The link between the risk-sensitive Bellman equation and reward regularisation has been investigated by \cite{petersen2000minimax,dupuis2000robust,hansen2011robustness,osogami2012robustness,bauerle2014more}.

\begin{proof}
The risk-sensitive Bellman equation (\autoref{eq:risk-sensitive}) reads
\begin{equation}
  V^{\ast}(s)
  \;=\; \max_{a \in \mathcal{A}}
        \left\{ r(s, a) + \frac{\gamma}{\beta_s}\,
        \log \mathbb{E}_{S' \sim K(\cdot \mid s, a)}\!\Big[\exp\big(\beta_s\, V^{\ast}(S')\big)\Big] \right\}.
  \label{eq:risk-recap}
\end{equation}
We treat the two regimes of $\beta_s$ separately. The two cases share
the same local return; only the $\textrm{sup}/\textrm{inf}$ alternation differs.

\textit{Risk-seeking case ($\beta_s > 0$).}
Applying \autoref{eq:gibbs} at each $(s, a)$ with
$\mu = p(\cdot | s, a)$, $f = V^{\ast}$, and $\beta = \beta_s$ leads to 
\begin{equation}
  \frac{1}{\beta_s}\log \mathbb{E}_{S' \sim K(\cdot \mid s, a)}\!\Big[\exp\big(\beta_s\, V^{\ast}(S')\big)\Big]
  \;=\; 
  \sup_{q \in \Delta(\mathcal{S})}
    \left\{ \mathbb{E}_{S' \sim q}\Big[V^{\ast}(S')\Big] - \frac{1}{\beta_s}\, D_{\mathrm{KL}}\!\big(q \,\big\|\, K(\cdot \mid s, a)\big) \right\}.
  \label{eq:risk-gibbs-pos}
\end{equation}
Substituting this expression into the right-hand side of \autoref{eq:risk-recap} leads to
\begin{align}
  V^{\ast}(s)
  &= \max_{a\in\mathcal A} \sup_{q\in \Delta(S)}
     \left\{ r(s, a) + \gamma\, \mathbb{E}_{S' \sim q}\big[V^{\ast}(S')\big]
             - \frac{\gamma}{\beta_s}\, D_{\mathrm{KL}}\!\big(q \,\|\, p(\cdot \mid s, a)\big) \right\} \nonumber \\
  &= \max_{a\in\mathcal A} \sup_{q\in \Delta(S)}
     \mathbb{E}_{S' \sim q}
     \!\Big[\, \ell_{\text{risk}}(s, a, S'; q) + \gamma\, V^{\ast}(S') \,\Big],
  \label{eq:risk-bellman-conjugate-pos}
\end{align}
where the risk-sensitive \emph{conjugate local return} is
\begin{equation}
    \ell_{\text{risk}}(s, a, s'; q)
    \;=\; r(s, a) \;+\; \frac{\gamma}{\beta_s}\,
          \log \frac{K(s' \mid s, a)}{q(s')}.
  \label{eq:ell-risk}
\end{equation}
Using \autoref{eq:gibbs-optimal} one can find that the distribution that extremises \autoref{eq:risk-bellman-conjugate-pos}, $q^{\ast}$, is the Gibbs tilt given by
\begin{equation}
  q^{\ast}(s' \mid s, a)
  \;=\; \frac{p(s' | s, a)\,\exp(\beta_s\, V^{\ast}(s'))}
             {\mathbb{E}_{s'' \sim K(\cdot \mid s, a)}[\exp(\beta_s\, V^{\ast}(s''))]},
\end{equation}
which concentrates on \emph{high}-value next states.

\textit{Risk-averse case ($\beta_s < 0$).}
The same calculation goes through with the supremum replaced by an
infimum. For $\beta_s < 0$, the Gibbs identity reads
\begin{equation}
  \frac{1}{\beta_s}\log \mathbb{E}_{S' \sim K(\cdot|s,a)}\!\big[\exp(\beta_s\, V^{\ast}(S'))\big]
  =\! \inf_{q \in \Delta(\mathcal{S})}
        \left\{ \mathbb{E}_{S' \sim q}[V^{\ast}(S')] - \frac{1}{\beta_s}\, D_{\mathrm{KL}}\!\big(q \,\|\, K(\cdot \mid s, a)\big) \right\},
  \label{eq:risk-gibbs-neg}
\end{equation}
where $-1/\beta_s = 1/|\beta_s| > 0$, so that the KL term still enters
with a non-negative coefficient. Substituting this into \autoref{eq:risk-recap} yields
\begin{equation}
  V^{\ast}(s)
  \;=\; \max_{a\in\mathcal A} \inf_{q\in \Delta(S)}
        \mathbb{E}_{S'\sim q}
        \!\Big[\, \ell_{\text{risk}}(s, a, S'; q) + \gamma\, V^{\ast}(S') \,\Big],
  \label{eq:risk-bellman-conjugate-neg}
\end{equation}
with the \emph{same} local return as \autoref{eq:ell-risk}. There, the sign of $\beta_s$ is preserved inside $\ell_{\text{risk}}$, so that the coefficient $-\gamma/\beta_s = \gamma/|\beta_s|$ in front of
$\log(q/p)$ is positive, and the KL penalty $\gamma|\beta_s|^{-1} D_{\mathrm{KL}}(q\,\|\,p)$ is non-negative. The distribution that extremises \autoref{eq:risk-bellman-conjugate-neg}, $q^{\ast}$, is again a Gibbs tilt, but now it is concentrated on low-value next states --- the agent maximises against an adversarial nature whose deviations from $p$ are KL-penalised.

\end{proof}

\subsection{Proof of \autoref{eq:robust_bellman_equation} --- the robust Bellman equation}
\label{app:proof-robust}

The link between the robust Bellman equation and reward regularisation has been investigated by \cite{husain2021regularized,derman2021twice}.

\begin{proof}
The robust Bellman equation (\autoref{eq:robust_bellman_equation}) reads
\begin{equation}
  V^{\ast}(s)
  \;=\; \max_{a \in \mathcal{A}}
        \left\{ r(s, a) + \gamma \min_{q \in \mathcal{Q}(s, a)}
        \mathbb{E}_{S' \sim q}\Big[V^{\ast}(S')\Big] \right\},
  \qquad \mathcal{Q}(s, a) \subseteq \Delta(\mathcal{S}),
  \label{eq:robust-recap}
\end{equation}
where $\mathcal{Q}(s, a)$ is the (rectangular) ambiguity set at $(s, a)$. 
Let us introduce the convex-analytic indicator
\begin{equation}
  \mathcal{I}_{\mathcal{Q}(s,a)}(q)
  \;=\; \begin{cases} 0 & \text{if } q \in \mathcal{Q}(s, a), \\
                       +\infty & \text{if } q \notin \mathcal{Q}(s, a). \end{cases}
\end{equation}
For any function $V : \mathcal{S} \to \mathbb{R}$,
\begin{equation}
  \min_{q \in \mathcal{Q}(s, a)} \mathbb{E}_{S' \sim q}\Big[V(S')\Big]
  \;=\; \inf_{q \in \Delta(\mathcal{S})}
        \Big\{ \mathbb{E}_{S'\sim q}\Big[V(S')\Big] + \mathcal{I}_{\mathcal{Q}(s,a)}(q) \Big\},
  \label{eq:hard-indicator}
\end{equation}
because any $q \notin \mathcal{Q}(s, a)$ contributes $+\infty$ to the
right-hand side and is excluded from the infimum. Substituting
\autoref{eq:hard-indicator} into \autoref{eq:robust-recap} yields
\begin{align}
  V^{\ast}(s)
  &= \max_{a\in\mathcal A} \inf_{q \in \Delta(\mathcal{S})}
     \Big\{ r(s, a) 
     + \gamma\, \mathbb{E}_{S' \sim q}\Big[V^{\ast}(S')\Big]
     + \gamma\, \mathcal{I}_{\mathcal{Q}(s, a)}(q) \Big\} \nonumber \\
  &= \max_{a\in\mathcal A} \inf_{q \in \Delta(\mathcal{S})}
     \mathbb{E}_{S'\sim q}
     \!\Big[\, \ell_{\text{robust}}(s, a, S'; q) + \gamma\, V^{\ast}(S') \,\Big],
  \label{eq:robust-bellman-conjugate}
\end{align}
where the robust \emph{conjugate local return} is
\begin{equation}
    \ell_{\text{rob}}(s, a, s'; q)
    \;=\; r(s, a) \;+\; \mathcal{I}_{\mathcal{Q}(s, a)}(q).
  \label{eq:ell-robust-corrected}
\end{equation}
Above, the $\gamma$ multiplying the indicator has been absorbed into
$\mathcal{I}_{\mathcal{Q}}$ (since $\gamma\cdot(+\infty)=+\infty$ and
$\gamma\cdot 0=0$). The distribution extremising \autoref{eq:robust-bellman-conjugate} at $(s, a)$, $q^{\ast}$, is any
element of $\mathcal{Q}(s, a)$ that attains
$\min_{q\in\mathcal{Q}}\mathbb{E}_q[V^{\ast}]$.
\end{proof}

Note that the robust Bellman equation is the singular limit $\beta_s \to -\infty$ of the risk-sensitive Bellman equation under a KL-ball ambiguity set. 
Also, note that some of these different regularisations can be combined; for instance, one can add KL-like terms over policies and transitions to obtain soft risk-sensitive solutions~\citep{derman2021twice}. 

\section{Distortion aggregators as deformed expectations}
\label{app:distortion-deformation}

This appendix proves that aggregators that can be written as \autoref{eq:choquet_aggregator} can be expressed as an ordinary expectation of the same return $\mathbf J$ under deformed dynamics $\mathbf Q$, and gives a general formula for $\mathbf Q$ in terms of the distortion function $g$. 

For this purpose, let us focus on the case of real returns $\mathbf J(H)\in\mathbb R$ where $H\sim \mathbf P_\pi$. 
The return cumulative distribution
function, survival function, and lower quantile function are given by
\begin{equation}
F_{\mathbf J}(x)=P_\pi(\mathbf J\le x),\qquad
S_{\mathbf J}(x)=P_\pi(\mathbf J> x)=1-F_{\mathbf J}(x),\qquad
F_{\mathbf J}^{-1}(u)=\inf\{x:\ F_{\mathbf J}(x)\ge u\}.
\end{equation}
A \emph{distortion} is an increasing function $g:[0,1]\to[0,1]$ with $g(0)=0$ and $g(1)=1$.
The associated aggregator is the Choquet integral of $\mathbf J$ against the capacity
$A\mapsto g(P_\pi(A))$,
\begin{equation}
\tilde{\mathbf{F}}^g[\mathbf J]
:=\int \mathbf J\,\mathrm{d}(g\circ \mathbf P_\pi)
=\int_0^{\infty} g\big(S_{\mathbf J}(x)\big)\,\mathrm{d}x
-\int_{-\infty}^{0}\Big(1-g\big(S_{\mathbf J}(x)\big)\Big)\,\mathrm{d}x .
\label{eq:choquet-def}
\end{equation}
We will use the bounded-support form of~\autoref{eq:choquet-def}: if $a\le \mathbf J\le b$ almost surely, then the translation equivariance of the Choquet integral\footnote{That is, the fact that $\mathbf{F}_g[c+Z]=c+\mathbf{F}_g[Z]$, which holds because $g\circ \mathbf P_\pi$ is normalised.} together with the tail-integral formula\footnote{The tail-integral formula says
$\mathbb E[Y]\;=\;\int_0^\infty \mathbb P(Y>x)\,\mathrm dx\;=\;\int_0^\infty S_Y(x)\,\mathrm dx$ with $S_Y(x)=\mathbb P(Y>x)$.} for the non-negative variable $\mathbf J-a$ gives
\begin{equation}
\tilde{\mathbf{F}}^g[\mathbf J]\;=\;a+\int_a^b g\big(S_{\mathbf J}(x)\big)\,\mathrm{d}x .
\label{eq:choquet-bounded}
\end{equation}

\subsection{The smooth case: an explicit density}

\begin{proposition}[Deformed-measure representation]
\label{prop:density}
Assume $g\in C^1([0,1])$ (so $g'\ge 0$) and that $\mathbf J$ is bounded with continuous distribution
function $F_{\mathbf J}$. Define a measure $\mathbf Q$ on $\mathcal{T}$ by
\begin{equation}
\frac{\mathrm{d}\mathbf Q}{\mathrm{d}\mathbf P_\pi}
\;=\;
g'\big(S_{\mathbf J}(\mathbf J)\big),
\qquad
S_{\mathbf J}(\mathbf J):\ \omega\mapsto S_{\mathbf J}\!\big(\mathbf J(\omega)\big)=1-F_{\mathbf J}\!\big(\mathbf J(\omega)\big).
\label{eq:density}
\end{equation}
Then $\mathbf Q$ is a probability measure on $\mathcal{T}$, $\mathbf Q\ll \mathbf P_\pi$, and
\begin{equation}
\tilde{\mathbf{F}}_g[\mathbf J]\;=\;\mathbb{E}_{\mathbf Q}[\mathbf J]\;=\;\mathbb{E}_{\mathbf P_\pi}\!\big[\mathbf J\,g'(S_{\mathbf J}(\mathbf J))\big].
\label{eq:claim-density}
\end{equation}
\end{proposition}

\begin{proof}
Let us first prove that $\mathbf Q$ is a probability measure. 
Because $F_{\mathbf J}$ is continuous, the probability integral transform gives
$U:=F_{\mathbf J}(\mathbf J)\sim\mathrm{Unif}[0,1]$, hence $S_{\mathbf J}(\mathbf J)=1-U\sim\mathrm{Unif}[0,1]$. The density
$g'(S_{\mathbf J}(\mathbf J))$ is non-negative ($g$ increasing) and
\begin{equation}
\mathbb{E}_{\mathbf P_\pi}\!\big[g'(S_{\mathbf J}(\mathbf J))\big]
=\int_0^1 g'(v)\,\mathrm{d}v
=g(1)-g(0)=1 .
\end{equation}
Thus, $\mathbf Q$ defined by~\autoref{eq:density} is a probability measure absolutely continuous with respect to $\mathbf P_\pi$.

As a second step, let us show that 
\begin{equation}
\tilde{\mathbf{F}}^g[\mathbf J]\;=\;\int_0^1 F_{\mathbf J}^{-1}(p)\,g'(1-p)\,\mathrm{d}p .
\label{eq:master}
\end{equation}
Let $a\le \mathbf J\le b$. Starting from~\autoref{eq:choquet-bounded} and substituting $p=F_{\mathbf J}(x)$ results in 
\begin{equation}
\int_a^b g\big(S_{\mathbf J}(x)\big)\,\mathrm{d}x
=\int_a^b g\big(1-F_{\mathbf J}(x)\big)\,\mathrm{d}x
=\int_0^1 g(1-p)\,\mathrm{d}F_{\mathbf J}^{-1}(p).
\end{equation}
Integrating by parts, with $\mathrm{d}\!\left[g(1-p)\right]=-g'(1-p)\,\mathrm{d}p$,
\begin{equation}
\int_0^1 g(1-p)\,\mathrm{d}F_{\mathbf J}^{-1}(p)
=\Big[g(1-p)\,F_{\mathbf J}^{-1}(p)\Big]_{0}^{1}
+\int_0^1 F_{\mathbf J}^{-1}(p)\,g'(1-p)\,\mathrm{d}p .
\end{equation}
The boundary term is $g(0)F_{\mathbf J}^{-1}(1)-g(1)F_{\mathbf J}^{-1}(0)=0\cdot b-1\cdot a=-a$. Substituting
into~\eqref{eq:choquet-bounded},
\begin{equation}
\tilde{\mathbf{F}}^g[\mathbf J]=a+\Big(-a+\int_0^1 F_{\mathbf J}^{-1}(p)\,g'(1-p)\,\mathrm{d}p\Big)
=\int_0^1 F_{\mathbf J}^{-1}(p)\,g'(1-p)\,\mathrm{d}p,
\end{equation}
proving \autoref{eq:master}.

To conclude the proof, couple $\mathbf J$ to its own quantile function through $U=F_{\mathbf J}(\mathbf J)\sim\mathrm{Unif}[0,1]$, so that
$\mathbf J=F_{\mathbf J}^{-1}(U)$ almost surely and $S_{\mathbf J}(\mathbf J)=1-U$. Then
\begin{equation}
\mathbb{E}_{\mathbf P_\pi}\!\big[\mathbf J\,g'(S_{\mathbf J}(\mathbf J))\big]
=\mathbb{E}\big[F_{\mathbf J}^{-1}(U)\,g'(1-U)\big]
=\int_0^1 F_{\mathbf J}^{-1}(u)\,g'(1-u)\,\mathrm{d}u .
\end{equation}
By~\autoref{eq:master}, this equals $\tilde{\mathbf{F}}^g[\mathbf J]$, proving~\autoref{eq:claim-density}.
\end{proof}

\autoref{eq:density} reweights each outcome by $g'$ evaluated at that outcome's own
upper-tail rank $S_{\mathbf J}(\mathbf J)\in[0,1]$. Outcomes in the lower tail of $\mathbf J$ and outcomes in the upper tail are up- or down-weighted according to the
local slope of $g$ there. 

The representation provided in \autoref{eq:density} exists for every distortion function $g$. 
When $g$ is not differentiable (e.g.\ the kinked $g$ of \textsc{CVaR}) or $\mathbf J$ has atoms, then \autoref{eq:density} must be replaced by a Lebesgue--Stieltjes construction. For technical details, please see \cite{follmer2025stochastic,dhaene2006risk}.

\subsection{Example: conditional value at risk}

Take the worst-tail CVaR at level $\alpha\in(0,1)$, i.e.\ the average of the lowest
$\alpha$-fraction of returns. Its distortion and (a.e.) derivative are
\begin{equation}
g(s)=\tfrac{1}{\alpha}\big(s-(1-\alpha)\big)^{+},
\qquad
g'(s)=\tfrac{1}{\alpha}\,\mathbf{1}\{s>1-\alpha\},
\end{equation}
so that $g(1-\alpha)=0$, $g(1)=1$. The density in \autoref{eq:density} becomes
\begin{equation}
\frac{\mathrm{d}\mathbf Q}{\mathrm{d}\mathbf P_\pi}
=\tfrac{1}{\alpha}\,\mathbf{1}\big\{S_{\mathbf J}(\mathbf J)>1-\alpha\big\}
=\tfrac{1}{\alpha}\,\mathbf{1}\big\{F_{\mathbf J}(\mathbf J)<\alpha\big\}
=\tfrac{1}{\alpha}\,\mathbf{1}\big\{\mathbf J\le q_\alpha(\mathbf J)\big\},
\end{equation}
the reweighting onto the worst $\alpha$-tail used in the main text. The kink of $g$ at
$s=1-\alpha$ is harmless when $F_J$ is continuous, since $S_J(J)\sim\mathrm{Unif}[0,1]$ places
no mass at the single point $1-\alpha$. 

\section{Related work}
\label{sec:related-work}

The framework developed in this paper synthesises and extends a
substantial body of prior work across reinforcement learning,
optimal control, statistics, theoretical neuroscience, and category
theory. \autoref{sec:three-fold} discusses various formal antecedents, including active inference, control-as-inference, the soft Bellman equation, risk-sensitive, robust, and entropy-regularised families. In this section, we situate the framework within other research programmes from which these results are
drawn, addressing those aspects of the literature that are
conceptual or programmatic rather than formal, and
indicating in each case where our contribution relates to existing work.

\subsection{Connections to related reinforcement learning frameworks}

The ideas developed in this paper are complementary to several existing research directions that modify, extend, or generalise Bellman recursion. 

For instance, classic work on general value functions \citep{sutton2011horde} extends value functions beyond cumulative reward to arbitrary predictive questions specified by cumulants and termination conditions. From the perspective of the present work, this corresponds to changing the object propagated through Bellman recursion while preserving its recursive structure. 
Reward machines \citep{icarte2018using} propose to augment the environment by including automaton states that encode task progress, thereby transforming non-Markovian reward specifications into Markovian ones. Reward machines can be seen as an example of restoring \condR through state augmentation.

\cite{RePEc} study which aggregation operators admit recursive dynamic-programming formulations, while \cite{tang2025recursive} develop a general framework for recursively aggregating rewards beyond the additive discounted sum. These works investigate which reward aggregators admit recursive dynamic programming formulations; by contrast, aggregation duality assumes such a recursive formulation exists and studies equivalence transformations between different Bellman recursions.

\cite{oh2020discovering} propose to automatically discover Bellman-style update rules through automatic search over Bellman update rules, providing an interesting complement to the view taken in this work. Their work treats the Bellman backup itself as an object to be optimised over a space of candidate operators. 
Finally, \cite{rahme2019theoretical} reformulates reinforcement learning in terms of partition functions and free energies rather than conventional value functions. By doing this, this work is closely related to object duality.

\subsection{Bayes-adaptive RL, empowerment, and information-theoretic objectives}
\label{subsec:rw-bayes-adaptive}

A related literature treats epistemic (belief-related) and pragmatic
(reward-related) value symmetrically by absorbing exploration into
the reward structure. \emph{Bayes-adaptive MDPs} (BAMDPs), introduced
by \citet{Duff2002} and developed by
\citet{GhavamzadehMannorPineauTamar2015,ZintgrafShialaIgl2020VariBAD},
lift the planning problem to belief-space so that the agent's
posterior over MDP parameters becomes part of the state.
\emph{Information-directed sampling}
\citep{RussoVanRoy2014,RussoVanRoy2018} trades off regret and
information gain in a unified objective. \emph{Empowerment},
introduced by \citet{KlyubinPolaniNehaniv2005} and developed by
\citet{SalgeGlackinPolani2014}, defines an information-theoretic
intrinsic reward as the channel capacity from actions to future
states. Related lines of work include \emph{curiosity-driven exploration}
\citep{PathakAgrawalEfros2017} and \emph{intrinsic motivation}
\citep{OudeyerKaplan2009}.

These approaches operationalise specific instances of the
belief--reward symmetry without developing it as a structural
principle. Empowerment is a particularly clean example: the reward
is defined as a functional of the belief over future states, which
is a direct instance of the object-swap duality applied to a
particular class of objectives. BAMDPs are instances of state
augmentation restoring \condD on a posterior-augmented state. Our
framework provides the structural language in which these
constructions can be understood as choices of
$(\mathbf{J}, \mathbf{F})$ pairs satisfying \condR and \condA with appropriate state augmentation where necessary.
We do not derive any new exploration algorithm, and view this
literature as a source of concrete instantiations of our framework
rather than a competitor to it.

\subsection{Categorical approaches}
\label{subsec:rw-categorical}

A line of work treats Markov decision processes from the perspective
of universal algebra, coalgebra, and category theory.
\citet{FeysHansenMoss2018} formalise an MDP as a coalgebra
$m = \langle u, t \rangle : S \to \mathbb{R} \times (\Delta S)^A$,
combining reward and transition functions, and develop a coinductive
theory of long-term value. Their two main contributions are a new
proof principle they call \emph{contraction coinduction}, based on
Banach's fixpoint theorem, and a characterisation of the discounted
value function as the unique morphism into a final coalgebra via a
generalised notion of \emph{corecursive algebra}. 

A more recent and rapidly developing research programme is the
\emph{categorical cybernetics} approach to RL \citep{HedgesRodriguezSakamoto2023ValueIteration,HedgesRodriguezSakamoto2024RL}.
The starting observation is that value iteration can be represented
as precomposition with a specific \emph{optic} --- a categorical
structure originating in lens theory 
\citep{ClarkeElkinsHedges2024Optics,Riley2018Optics}. These works extend this 
to parametrised optics that apply to action-value
functions and depend on samples, treating major RL algorithms
(value iteration, policy iteration, Q-learning, deep Q-learning,
actor-critic) as instances of a categorical construction in
which the Bellman update is the backward pass of an optic
representing the agent--environment interaction. 

The categorical literature provides what is, in principle, the
natural mathematical home for the framework developed in this paper.
The conditions D, R, and A that we formulate are the kinds of structures that Markov categories
\citep{Fritz2020Synthetic} can naturally accommodate. The recursive
decomposition of value is related to the corecursive-algebra construction of
\citet{FeysHansenMoss2018}. 
Our framework should be
understood as a non-categorical presentation of structures that
admit a categorical formulation. 
Our contribution relative to this literature is twofold. First, the
existing categorical literature treats one or another of the three
dualities at a time: for example, \citet{FeysHansenMoss2018} develop the
coalgebraic structure of the dynamics--reward pair without
addressing the aggregation choice. 
Second, our framework is formulated in elementary terms, 
which makes it accessible to
the broader optimal-control and RL communities while remaining
compatible with the categorical formulations. We view the explicit
categorical development of the triadic synthesis as a promising direction for future work  
--- in particular, an articulation of the universal property satisfied by the Bellman
state $\phi$ when all three conditions hold simultaneously.

\end{document}